\documentclass{article} 
\usepackage{arxiv_preprint,times}

\usepackage{amsmath,amsfonts,bm}

\def\eqref#1{equation~\ref{#1}}

\def\1{\bm{1}}

\DeclareMathAlphabet{\mathsfit}{\encodingdefault}{\sfdefault}{m}{sl}
\SetMathAlphabet{\mathsfit}{bold}{\encodingdefault}{\sfdefault}{bx}{n}

\usepackage{amssymb}
\usepackage{amsthm}
\usepackage{float}
\usepackage{hyperref}
\hypersetup{hidelinks, pdftitle={Regularizing Modality Contribution Drift in Multimodal Continual Learning}, pdfauthor={ZhenZhang, Jielei Chu, Wenjie Ban, Tian Sang, Yuxiao Li, Bin Liu, Fengmao Lv, Tianrui Li}}
\usepackage{url}
\usepackage{graphicx}
\usepackage{subcaption}
\usepackage{wrapfig}
\usepackage{booktabs}
\usepackage[table]{xcolor}
\usepackage{colortbl}
\usepackage{tcolorbox}
\usepackage{tikz}
\usetikzlibrary{calc}
\tcbuselibrary{breakable}
\definecolor{cmcdrgreen}{RGB}{232,246,238}
\definecolor{cmcdracc}{RGB}{150,45,54}
\definecolor{cmcdrforget}{RGB}{23,118,111}
\definecolor{propboxbg}{RGB}{244,249,246}
\definecolor{propboxrule}{RGB}{145,171,157}
\newtcolorbox{restatedproposition}{
  breakable,
  colback=propboxbg,
  colframe=propboxrule,
  boxrule=0.45pt,
  arc=1.5pt,
  left=6pt,
  right=6pt,
  top=5pt,
  bottom=5pt,
  before skip=7pt,
  after skip=9pt
}
\newcommand{\accdelta}[1]{\textcolor{cmcdracc}{\textbf{(+#1)}}}
\newcommand{\downdelta}[1]{\textcolor{cmcdrforget}{\textbf{(-#1)}}}
\newcommand{\stddev}[1]{\,\ensuremath{\pm\,#1}}
\newtheorem{proposition}{Proposition}
\newtheorem{lemma}[proposition]{Lemma}
\newtheorem{corollary}[proposition]{Corollary}

\title{Regularizing Modality Contribution Drift in Multimodal Continual Learning}

\author{ZhenZhang\textsuperscript{1}, Jielei Chu\textsuperscript{1}, Wenjie Ban\textsuperscript{1}, Tian Sang\textsuperscript{1}\\[4pt]
Yuxiao Li\textsuperscript{2}, Bin Liu\textsuperscript{1}, Fengmao Lv\textsuperscript{1}, Tianrui Li\textsuperscript{1}\\[9pt]
{\normalsize\textsuperscript{1}Southwest Jiaotong University}\\[2pt]
{\normalsize\textsuperscript{2}West China Hospital, Sichuan University}}
\date{}

\begin{document}

\maketitle

\begin{abstract}
Multimodal continual learning (MMCL) aims to acquire new knowledge from multimodal data while retaining previously 
learned knowledge. Existing MMCL methods primarily mitigate forgetting by aligning cross-modal representations or preserving feature-level semantic similarity. However, \textbf{different tasks may rely on different modalities, and learning new tasks can alter how modalities contribute to predictions on previously learned tasks}.
It remains \textbf{underexplored how modality contributions evolve across
 incremental stages and how such changes relate to forgetting in MMCL}.
  We term such changes Modality Contribution Drift (MCD) and introduce an MCD score based on controlled modality-subset interventions. 
  Our theoretical and empirical analyses show how \textbf{contribution drift can lead 
  to forgetting, while existing MMCL and conventional CL methods do not effectively mitigate MCD.} To address this issue, we propose Continual Modality 
   Contribution Drift Regularization
     (CMCDR) to preserve the modality contribution profiles of previously learned tasks in both replay-based and
      replay-free settings. In the replay-based setting, CMCDR estimates modality contributions from stored old samples and 
      regularizes their drift relative to a frozen previous model. In the replay-free setting, CMCDR uses current-task samples 
      as probes to match old-class contribution profiles between the current and frozen models without storing old exemplars.
       Across six benchmarks covering multimodal class-incremental learning and continual multimodal question answering,
        CMCDR substantially reduces modality contribution drift, improves average accuracy by $0.88\mathord{\sim}7.66$ percentage points, and reduces average forgetting by $0.85\mathord{\sim}11.55$ percentage points over the corresponding baselines.

\end{abstract}

\section{Introduction}

Multimodal continual learning (MMCL) aims to learn from a sequence 
of multimodal tasks while retaining knowledge acquired from previous
 tasks. A key challenge in mitigating catastrophic forgetting in MMCL
  is to preserve cross-modal relationships established during earlier
   stages. Existing MMCL methods primarily address this challenge by 
   stabilizing multimodal representations and preserving cross-modal 
   semantic relationships. For example, AV-CIL~\citep{pian2023avcil} 
   aligns audio-visual representations by preserving instance- and 
   class-aware semantic similarity, while MG-CLIP~\citep{huang2025mindgap}
    preserves the image--text modality gap across incremental stages. However, 
    stabilizing multimodal representations does not necessarily preserve how
     individual modalities and cross-modal interactions contribute to model predictions.

This issue is particularly important in MMCL because different tasks can
 rely on substantially different combinations of modality-specific and 
 cross-modal information~\citep{hessel2020emap,liang2023quantifying}. 
 Figure~\ref{motivation}(a) illustrates such task-dependent modality
reliance, e.g., the road-crossing task relies predominantly on visual cues,
whereas the question-answering task requires joint reasoning over visual
and linguistic information. Our diagnostic analyses further reveal substantial
    variation in modality contributions across incremental tasks within the same
     dataset (see Figure~\ref{motivation}(b)) and across multimodal benchmarks (see Figure~\ref{motivation}(c)),
    showing that modality reliance can vary substantially across tasks and learning settings.

\begin{figure}[htbp]
\vspace{-0.9em}
\centering
\includegraphics[width=0.99\textwidth]{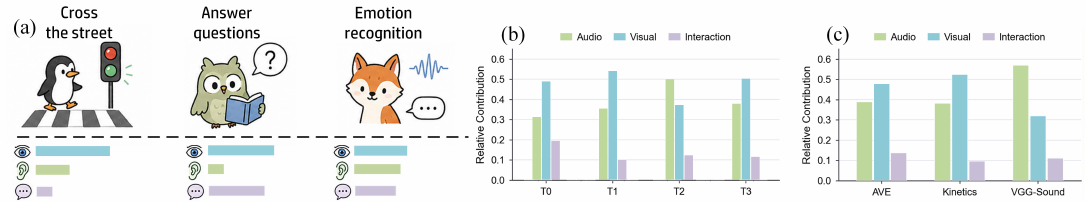}
\vspace{-0.5em}
\caption{Task-dependent modality contributions.
(a) Illustrative examples.
(b) Contribution profiles obtained by partitioning AVE~\citep{tian2018ave} categories into four tasks and estimating audio, visual, and cross-modal contributions.
(c) Average contribution profiles computed on AVE, Kinetics-Sounds~\citep{arandjelovic2017look}, and VGGSound~\citep{chen2020vggsound}.}
\label{motivation}
\vspace{-0.3em}
\end{figure}

In MMCL, adapting to such task-dependent modality reliance can alter the contributions of modality-specific and cross-modal information to model predictions. These changes can alter the contribution profiles 
   established for previous tasks and lead to forgetting.
However, the evolution of modality contributions across incremental stages and their relation to forgetting in MMCL remain underexplored.

To characterize these contribution changes, we introduce an intervention-based 
estimator that quantifies the contributions of individual modalities and cross-modal
 interactions to model predictions. Based on these estimates, we define an MCD score 
 to quantify changes in both the absolute magnitude and relative composition of old-task
  contribution profiles across incremental stages. Our theoretical and empirical
   analyses show how contribution drift can lead to forgetting 
and further show that existing MMCL and conventional CL methods do not effectively 
mitigate MCD. Therefore, we propose Continual Modality Contribution Drift Regularization (CMCDR), which regularizes
     changes in old-task contribution profiles. Our main contributions are summarized as follows:

\begin{list}{$\bullet$}{\setlength{\leftmargin}{1.2em}
\setlength{\itemsep}{0.15em}
\setlength{\topsep}{0.2em}
\setlength{\parsep}{0pt}
\setlength{\partopsep}{0pt}}

\item \textbf{Problem and Metric.} We identify Modality Contribution Drift (MCD), an underexplored 
phenomenon in MMCL, and introduce an intervention-based contribution estimator 
together with an MCD metric to quantify changes in contribution profiles across incremental stages.

\item \textbf{Method.} We propose Continual Modality Contribution Drift Regularization (CMCDR),
 with replay-based and replay-free variants, to mitigate contribution drift and reduce forgetting.

\item \textbf{Performance.} We conduct extensive experiments on
 multimodal class-incremental learning and continual multimodal 
 question answering, showing that CMCDR consistently reduces MCD 
 and improves continual-learning performance across MMCL settings and benchmarks.

\end{list}

\section{Contribution drift in MMCL}
\label{sec:contribution_drift}
\vspace{-0.5em}

We use modality-subset interventions
to study modality contribution drift (MCD)
in MMCL (Section~\ref{sec:causal_contribution}) and analyze how such drift can alter
old-task prediction margins and induce forgetting (Section~\ref{sec:mcd_forgetting}).
Theoretical and empirical analyses further show that
neither existing MMCL methods nor conventional continual
learning methods effectively mitigate modality contribution drift
(Sections~\ref{sec:mcd_forgetting} and~\ref{sec:representation_stability}).

\subsection{Intervention-Based Contribution Estimation}
\label{sec:causal_contribution}

We estimate modality contributions by intervening on modality subsets and
decomposing the induced predictive benefit. Let $\mathcal{M}=\{1,\ldots,M\}$ denote the set of modality indices, and
$\mathbf{x}_i=(x_i^1,\ldots,x_i^M)$ denote a multimodal sample.
We use $F_\theta$ to denote the full predictor, comprising the modality
encoders, fusion module, and classifier. For any $\mathcal S\subseteq\mathcal M$, we use $\bot^m$ to denote
the absence of modality $m$ and define the modified input and
corresponding model output as
\begin{equation}
\widetilde{x}_i^{m,\mathcal S}
=
\begin{cases}
x_i^m, & m\in\mathcal S,\\[-1pt]
\bot^m, & m\notin\mathcal S,
\end{cases}
\quad
\mathbf{x}_i^{\mathcal S}
 :=\mathcal R_{\mathcal S}(\mathbf{x}_i)
=\bigl(\widetilde{x}_i^{m,\mathcal S}\bigr)_{m=1}^{M},
\quad
\mathbf z_i(\mathcal S)
=F_\theta\!\left(\mathbf x_i^{\mathcal S}\right).
\label{eq:modality_intervention}
\end{equation}
Here, $\bot^m$ denotes the removal of modality $m$.
\textbf{Appendix~\ref{app:modality_removal}} describes how modality removal
is implemented for each architecture.

To quantify how each modality subset supports 
the prediction of a target class, we first define a scalar benefit
 from the corresponding logits. For a candidate set $\mathcal A$
  satisfying $a\in\mathcal A$ and $|\mathcal A|\geq2$, the target-class benefit is
\begin{equation}
B_i^{a,\mathcal{A}}(\mathcal{S})
=
z_{i,a}(\mathcal{S})
-
\log
\sum_{\ell\in\mathcal{A},\,\ell\neq a}
\exp\bigl(z_{i,\ell}(\mathcal{S})\bigr),
\label{eq:benefit_function}
\end{equation}
where $\mathbf{z}_i(\mathcal{S})$ is the logit vector for the modality subset, $z_{i,a}(\mathcal{S})$ is the target-class logit, and
$z_{i,\ell}(\mathcal{S})$ are competing logits for classes
$\ell\in\mathcal{A}$ with $\ell\neq a$. Thus,
$B_i^{a,\mathcal{A}}(\mathcal{S})$ is the target-vs-competitor log-sum-exp margin
when only modalities in $\mathcal{S}$ are retained. 

Taking the no-modality intervention as the reference condition,
$B_i^{a,\mathcal{A}}(\mathcal{S})-B_i^{a,\mathcal{A}}(\emptyset)$
quantifies the decision-level benefit of retaining the modalities in
$\mathcal{S}$ on the log-odds margin scale. Here,
$B_i^{a,\mathcal{A}}(\emptyset)$ is the response when no modality is retained,
while $B_i^{a,\mathcal{A}}(\mathcal{M})$ is the response under the full
multimodal input.

Since the benefit of a modality subset can arise from both individual
modalities and their interactions, we use the M\"obius
transform~\citep{harsanyi1963} to decompose the subset-level benefit into
coalition-wise contributions. For each nonempty modality coalition
$\mathcal{T}\subseteq\mathcal{M}$, we define its contribution as:
\begin{equation}
C_i^{a,\mathcal{A}}(\mathcal{T})
=
\sum_{\mathcal{S}\subseteq\mathcal{T}}
(-1)^{|\mathcal{T}|-|\mathcal{S}|}
B_i^{a,\mathcal{A}}(\mathcal{S}),
\qquad
\mathbf{C}_i
=
\left[
C_i^{a,\mathcal{A}}(\mathcal{T})
\right]_{\emptyset\neq\mathcal{T}\subseteq\mathcal{M}} .
\label{eq:general_contribution}
\end{equation}
By M\"obius inversion, the contribution profile exactly decomposes the
full-input benefit relative to the no-modality response:
\begin{equation}
\sum_{\emptyset\neq\mathcal{T}\subseteq\mathcal{M}}
C_i^{a,\mathcal{A}}(\mathcal{T})
=
B_i^{a,\mathcal{A}}(\mathcal{M})
-
B_i^{a,\mathcal{A}}(\emptyset).
\label{eq:contribution_completeness}
\end{equation}

The M\"obius decomposition separates the subset-level benefits into
singleton modality contributions and higher-order cross-modal interaction
contributions. Appendix~\ref{app:estimator_analysis} provides the definition
of the interventional response $B$ and the derivation of the resulting
contribution profile.

\subsection{Analysis of Modality Contribution Drift}

We use the estimator above to quantify changes
in modality reliance across incremental tasks.
We first define the MCD score using class-level contribution profiles
(Section~\ref{sec:mcd_forgetting}) and then examine whether conventional
continual-learning mechanisms (Section~\ref{sec:mcd_forgetting}) and
multimodal representation alignment (Section~\ref{sec:representation_stability})
can effectively suppress this drift.

\subsubsection{Exploration Setup}

\begin{figure*}[htbp]
	\centering
	\includegraphics[width=0.95\textwidth]{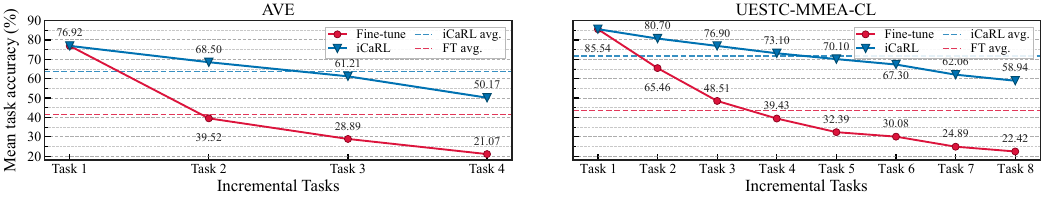}
	\caption{
    Mean task accuracy across incremental tasks in the exploratory diagnostic
    run on AVE and UESTC-MMEA-CL.
    }
	\label{zhexian1}
\end{figure*}
We use the fully connected classifier version of iCaRL~\citep{rebuffi2017icarl} (iCaRL-FC, hereafter referred to as iCaRL) as the baseline and evaluate it on two
representative MMCL datasets, AVE~\citep{tian2018ave} and
UESTC-MMEA-CL~\citep{xu2024uestc}. Experimental settings are provided in \textbf{Appendix~\ref{app:exploratory_setup}}. As shown in Figure~\ref{zhexian1}, iCaRL's accuracy declines consistently on both datasets as incremental learning proceeds. On AVE, the
mean task accuracy drops from 76.92\% at Task~1 to 50.17\% at Task~4,
with a four-stage average of 64.20\%.
On UESTC-MMEA-CL, it decreases from
85.54\% to 58.94\%, with an average accuracy of 71.83\%. Compared with
Fine-tune, iCaRL achieves higher mean task accuracy at later incremental stages, indicating that
 it alleviates catastrophic forgetting.

\subsubsection{Modality Contribution Drift in MMCL}
\label{sec:mcd_forgetting}

To assess whether iCaRL preserves old-class modality contributions,
we define an MCD score that combines changes in overall contribution
magnitude with weighted changes in the normalized contribution profile.

Let $\tau(y)$ denote the stage at which class $y$ is introduced, and
$\mathcal Y_{\leq s}$ the set of classes observed up to stage $s$.
We fix $\mathcal A_y=\mathcal Y_{\leq\tau(y)}$ so that contribution profiles
are compared under a fixed class context across stages.

For each task $k$, we maintain a separate probe set
$\mathcal P_k^{\mathrm{task}}$. After learning task $k$, we record the
contribution profiles of its probe samples as reference profiles.
For each $i\in\mathcal P_k^{\mathrm{task}}$ and stage $s\geq k$, let
$\mathbf c_{i,s}$ denote its contribution vector and
$\widehat{\mathbf c}_{i,s}$ its normalized counterpart:
\begin{equation}
\begin{aligned}
\mathbf c_{i,s}
&=
\left[
C_{i,s}^{y_i,\mathcal A_{y_i}}(\mathcal T)
\right]_{\emptyset\neq\mathcal T\subseteq\mathcal M},
&
\widehat{\mathbf c}_{i,s}
&=
\frac{\mathbf c_{i,s}}
{\|\mathbf c_{i,s}\|_1+\epsilon},
\end{aligned}
\label{eq:relative_contribution_vector}
\end{equation}
where each contribution term is computed using
Eq.~\ref{eq:general_contribution}, and $\epsilon>0$ stabilizes the
normalization and avoids degenerate zero-magnitude cases.

For $s\in\{k,t\}$, define
$m_{i,s}=\|\mathbf c_{i,s}\|_1$.
At stage $t$, let $\mathcal K_{<t}$ denote the set of all previous tasks.
The MCD score is defined as
\begin{equation}
\mathrm{MCD}_t
=
\frac{1}{|\mathcal K_{<t}|}
\sum_{k\in\mathcal K_{<t}}
\frac{1}{|\mathcal P_k^{\mathrm{task}}|}
\sum_{i\in\mathcal P_k^{\mathrm{task}}}
\left[
|m_{i,t}-m_{i,k}|
+
(\min(m_{i,t},m_{i,k})+\epsilon)
\|\widehat{\mathbf c}_{i,t}-\widehat{\mathbf c}_{i,k}\|_1
\right].
\label{eq:mcd}
\end{equation}
Its two terms measure changes in the overall contribution magnitude and
the normalized contribution profile, respectively.

\begin{wrapfigure}{r}{0.5\textwidth}
	\vspace{-1.2em}
	\centering
	\includegraphics[width=\linewidth]{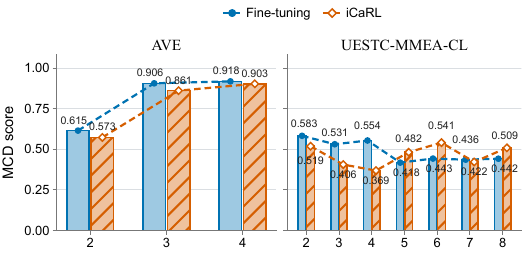}
\caption{MCD scores across incremental stages in the exploratory
diagnostic run on AVE and UESTC-MMEA-CL.}
	\label{zhexian2}
	\vspace{-1.0em}
\end{wrapfigure}

To assess whether conventional continual learning
can mitigate modality contribution drift, we compare
fine-tuning with iCaRL using the MCD score. As shown in Figure~\ref{zhexian2},
iCaRL only marginally reduces the average MCD compared with fine-tuning on both
AVE ($0.7790$ vs. $0.8130$) and UESTC-MMEA-CL ($0.4640$ vs. $0.4867$). Together
with the improved accuracy in Figure~\ref{zhexian1}, these results show that exemplar
replay and distillation can alleviate forgetting but do not effectively mitigate modality
contribution drift, suggesting that improved continual learning performance does not
necessarily imply more stable modality contributions.

\paragraph{Contribution Drift and Forgetting}

For an old sample $i$, fix a reference stage $r$ with
$\tau(y_i)\leq r<t$, and consider a candidate set
$\mathcal A_i\subseteq\mathcal Y_{\leq r}$ that contains $y_i$ and satisfies
$|\mathcal A_i|\geq 2$. We apply the same modality-removal intervention at
stages $r$ and $t$. All quantities below are computed with target $y_i$ and
candidate set $\mathcal A_i$, which we suppress from the notation.

Define
$\bar B_{i,s}(\mathcal S)=B_{i,s}(\mathcal S)-B_{i,s}(\emptyset)$.
The retained contribution margin is
\begin{equation}
m^{\mathrm{keep}}_{i,r\to t}
=
B_{i,t}(\emptyset)+\bar B_{i,r}(\mathcal M).
\label{eq:retained_contribution_margin}
\end{equation}

This quantity is the margin obtained by combining the stage-$t$ no-modality
response with the stage-$r$ contribution profile. Let $\mathbf c_{i,s}$ denote
the contribution vector at stage $s\in\{r,t\}$, and define its drift as
$\Delta\mathbf c_{i,r\to t}=\mathbf c_{i,t}-\mathbf c_{i,r}$.

\begin{proposition}[Contribution drift can cross the retention margin]
\label{prop:contribution_margin_loss}

Let $\gamma_{i,t}$ denote the target-class logit margin over the strongest
competitor in $\mathcal A_i$ at stage $t$. Then

\begin{equation}
\begin{aligned}
B_{i,t}(\mathcal M)
&=
m^{\mathrm{keep}}_{i,r\to t}
+\mathbf 1^\top\Delta\mathbf c_{i,r\to t},\\
\gamma_{i,t}
&\geq
m^{\mathrm{keep}}_{i,r\to t}
-\|\Delta\mathbf c_{i,r\to t}\|_1.
\end{aligned}
\label{eq:contribution_margin_bound}
\end{equation}

If $\|\Delta\mathbf c_{i,r\to t}\|_1<m^{\mathrm{keep}}_{i,r\to t}$,
then $y_i$ has a strictly higher logit than every other class in
$\mathcal A_i$. If $m^{\mathrm{keep}}_{i,r\to t}>0$ and
$\gamma_{i,t}\leq 0$, then
$\|\Delta\mathbf c_{i,r\to t}\|_1\geq
m^{\mathrm{keep}}_{i,r\to t}$. Thus, the retained margin gives a threshold below which prediction retention
is guaranteed and beyond which prediction failure becomes possible.

\end{proposition}

Proposition~\ref{prop:contribution_margin_loss} shows that contribution
drift can reduce the old-task prediction margin and cause forgetting.
The proof is given in \textbf{Appendix~\ref{app:contribution_forgetting_proof}}.

\subsubsection{Effect of Cross-modal Representation Stability on Contribution Drift}
\label{sec:representation_stability}

The preceding results show that replay and distillation do not necessarily
 preserve modality contribution structures. Existing MMCL methods primarily 
 stabilize cross-modal representations or semantic relationships across incremental
  stages~\citep{pian2023avcil,mo2023cign,pian2024contavsep,huang2025mindgap}.
   However, whether such representation-level stability also preserves modality 
   contributions to model decisions remains unclear. To examine this question, we
    augment iCaRL with Cross-modal Representation Alignment (RepAlign) from 
    AV-CIL~\citep{pian2023avcil} and the Cross-modal Similarity Distillation 
    Constraint (CrossSDC)~\citep{pian2024contavsep}. At each incremental step,
     we quantify cumulative representation drift from the stage at which each
      task is introduced to the current stage using Centered Kernel Alignment
       (CKA)~\citep{kornblith2019similarity}. Over the same interval and using
        the same probes, we measure cumulative contribution drift using Eq.~\ref{eq:mcd}.
         The results are averaged across incremental steps on AVE and UESTC-MMEA-CL. 
         Detailed experimental settings are provided in \textbf{Appendix~\ref{app:exploratory_setup}}.

\begin{figure}[htbp]
    \centering
    \begin{tikzpicture}
      \node[inner sep=0] (stabilityimage) {\includegraphics[width=0.96\columnwidth,trim=0 17 0 0,clip]{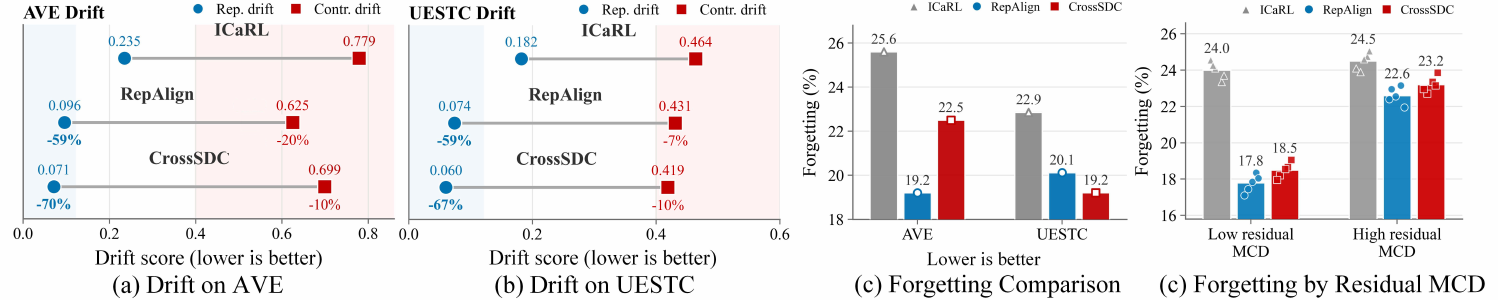}};
    \end{tikzpicture}
    \par\vspace{-0.35em}
    \makebox[\columnwidth][c]{%
      \makebox[0.235\columnwidth][c]{\scriptsize (a) Drift on AVE}%
      \makebox[0.235\columnwidth][c]{\scriptsize (b) Drift on UESTC}%
      \makebox[0.235\columnwidth][c]{\scriptsize (c) Forgetting comparison}%
      \makebox[0.235\columnwidth][c]{\scriptsize (d) Forgetting by residual MCD}%
    }
    \caption{Representation stability and residual MCD in MMCL. Both drift metrics measure cumulative changes from each task's introduction stage to the current stage.}
    \label{stability_panels}
    \vspace{-0.8em}
\end{figure}

Figure~\ref{stability_panels}(a,b) show that RepAlign and CrossSDC reduce representation 
drift by approximately 59\%$\mathord{\sim}$70\% on AVE and 59\%$\mathord{\sim}$67\% on UESTC-MMEA-CL, while achieving 
substantially smaller reductions in the MCD score. Figure~\ref{stability_panels}(c) further 
shows that both methods reduce average forgetting on the two datasets. 
In Figure~\ref{stability_panels}(d), we group old classes according to their residual
 MCD scores, with bars indicating group-level forgetting and points indicating class-level
  forgetting. Classes with low residual MCD scores exhibit clear reductions in forgetting,
   whereas those with high residual MCD scores retain forgetting levels comparable to iCaRL.
    Together, these results show that substantial reductions in representation drift can alleviate
     forgetting without effectively mitigating MCD, indicating that representation stability does
      not necessarily preserve modality contributions.

\textbf{Why representation stability is insufficient ?}
To formalize this observation, Proposition~\ref{prop:representation_insufficient}
 bounds contribution drift by changes in modality-subset representations and the
 decision map, showing why representation stability alone does not guarantee contribution stability.

\begin{proposition}[Representation stability does not guarantee contribution stability]
\label{prop:representation_insufficient}
For a sample $i$ with old-class target $a$ and stages $\tau(a)\leq r<t$, fix
$\mathcal{A}=\mathcal{A}_a$. Let $\mathbf u_{i,s}(\mathcal S)$ be the modality-subset representation
at stage $s$ in a common normed space and $g_s^{a,\mathcal A}$ the
corresponding target-vs-rest margin map, such that $B_{i,s}^{a,\mathcal A}(\mathcal S) =g_s^{a,\mathcal A}(\mathbf u_{i,s}(\mathcal S))$. Define
\begin{equation}
\begin{aligned}
\Delta_u(\mathcal S)
=\|\mathbf u_{i,t}(\mathcal S)-\mathbf u_{i,r}(\mathcal S)\|,
\Delta_g(\mathcal S)
=\left|g_t^{a,\mathcal A}(\mathbf u_{i,r}(\mathcal S))
-g_r^{a,\mathcal A}(\mathbf u_{i,r}(\mathcal S))\right|.
\end{aligned}
\label{eq:representation_decision_drift}
\end{equation}
If $g_t^{a,\mathcal A}$ is $L_g$-Lipschitz with respect to this norm, then
for any nonempty
$\mathcal{T}\subseteq\mathcal{M}$,
\begin{equation}
\left|
C_{i,t}^{a,\mathcal{A}}(\mathcal{T})
-
C_{i,r}^{a,\mathcal{A}}(\mathcal{T})
\right|
\leq
\sum_{\mathcal{S}\subseteq\mathcal{T}}
\left(
L_g\Delta_u(\mathcal{S})
+
\Delta_g(\mathcal{S})
\right).
\label{eq:representation_to_contribution_bound}
\end{equation}
Let $\bar B_{i,s}^{a,\mathcal A}(\mathcal S)
=B_{i,s}^{a,\mathcal A}(\mathcal S)-B_{i,s}^{a,\mathcal A}(\emptyset)$.
If every nonempty $\mathcal S\subseteq\mathcal T$ satisfies
$|\bar B_{i,t}^{a,\mathcal A}(\mathcal S)
-\bar B_{i,r}^{a,\mathcal A}(\mathcal S)|\leq\epsilon_{\bar B}$, then
$|C_{i,t}^{a,\mathcal{A}}(\mathcal{T})-C_{i,r}^{a,\mathcal{A}}(\mathcal{T})|
\leq (2^{|\mathcal{T}|}-1)\epsilon_{\bar B}$.
Moreover, there exist stage-invariant modality-subset representations and two
decision maps such that $\Delta_u(\mathcal S)=0$ for every $\mathcal S$, while
$C_{i,t}^{a,\mathcal A}(\mathcal T)\neq C_{i,r}^{a,\mathcal A}(\mathcal T)$
for some nonempty $\mathcal T$.
\end{proposition}
Proposition~\ref{prop:representation_insufficient} shows that
representation stability alone does not guarantee contribution stability because changes in the decision map can induce contribution drift
even when modality-subset representations are unchanged.
The proof is provided in
\textbf{Appendix~\ref{sub:proof_of_prop_representation_insufficient}}.

\section{Method}
\label{sec:method}

To control modality contribution drift, we propose Continual Modality
Contribution Drift Regularization (CMCDR), which regularizes old-class
contribution profiles by comparing the current model with a frozen previous
model under matched modality-subset interventions. When exemplars are available,
CMCDR matches contribution profiles on replayed old samples.
Without exemplars, it uses current-task samples as probes to match old-class
contribution profiles between the current and frozen models.

\subsection{Contribution-Profile Computation}

At incremental step $t>1$, let $F_{\theta_t}$, $F_{\theta_{t-1}}$, and
$\mathcal{Y}_{<t}$ be the current model, the frozen previous model,
and the old-class set, respectively. For any sample $\mathbf{x}_i$, stage $s\in\{t-1,t\}$, and
old-class output index $k\in\mathcal{Y}_{<t}$, CMCDR evaluates
$F_{\theta_s}$ using the interventions in Eq.~\ref{eq:modality_intervention}.
It then computes the contribution profile $\mathbf{C}^{k}_{i,s}$ over
all nonempty modality coalitions by setting
$(a,\mathcal{A})=(k,\mathcal{Y}_{<t})$ in
Eqs.~\ref{eq:benefit_function}--\ref{eq:general_contribution}. For replayed
old samples, $k=y_i$. In the replay-free case, $k$ ranges over all old-class
outputs in $\mathcal{Y}_{<t}$.

CMCDR compares current and frozen profiles with
\begin{equation}
\begin{aligned}
d_{\mathrm{cp}}(\mathbf{C},\mathbf{C}^{-})
={}&
\rho_{\delta}
\left(
\mathbf{C},
\operatorname{sg}(\mathbf{C}^{-})
\right)
+
\beta\,
\rho_{\delta}
\left(
\frac{\mathbf{C}}{\|\mathbf{C}\|_1+\epsilon},
\operatorname{sg}
\left(
\frac{\mathbf{C}^{-}}{\|\mathbf{C}^{-}\|_1+\epsilon}
\right)
\right),
\end{aligned}
\label{eq:cmcdr_profile_distance}
\end{equation}
where $\rho_{\delta}$ is the mean Smooth-$L_1$ loss with transition
$\delta>0$ defined in Eq.~\ref{eq:cmcdr_smooth_l1_definition},
$\operatorname{sg}(\cdot)$ stops gradients through the frozen model,
$\epsilon>0$, and $\beta>0$ balances absolute and relative profile matching.
The first term matches the raw contribution vectors, while the second matches
their normalized versions to preserve relative modality and interaction patterns.

\subsection{Replay-Based Contribution Constraint}

When old exemplars are available, CMCDR applies the above computation to
replayed samples $\mathcal{X}_{\mathrm{old}}\subseteq\mathcal{E}_{<t}$,
where $\mathcal{E}_{<t}$ denotes the exemplar buffer from previous tasks.
For each $(\mathbf{x}_i,y_i)\in\mathcal{X}_{\mathrm{old}}$, the profile index
is set to the ground-truth old label $k=y_i$.

The frozen previous model provides $\mathbf{C}^{y_i}_{i,t-1}$, and the
current model produces $\mathbf{C}^{y_i}_{i,t}$. The replay-based
objective is
\begin{equation}
\mathcal{L}^{\mathrm{rep}}_{\mathrm{cmcdr}}
=
\frac{1}{|\mathcal{X}_{\mathrm{old}}|}
\sum_{(\mathbf{x}_i,y_i)\in\mathcal{X}_{\mathrm{old}}}
d_{\mathrm{cp}}
\left(
\mathbf{C}^{y_i}_{i,t},
\mathbf{C}^{y_i}_{i,t-1}
\right).
\label{eq:replay_cmcdr_loss}
\end{equation}
Thus, replay-based CMCDR directly constrains old-sample contribution drift
at both the modality-specific and interaction levels.

\subsection{Replay-Free Contribution Constraint}

In the replay-free setting, no old-class samples are available. CMCDR uses
the $N_t=|\mathcal{X}_{\mathrm{new}}|$ current-task samples as probes
\citep{asadi2023prototype,li2017lwf} to compare the current and frozen models
under the same modality-subset interventions. For each probe $\mathbf{x}_i$,
the above computation yields $\mathbf{C}^{k}_{i,t}$ and
$\mathbf{C}^{k}_{i,t-1}$ for every $k\in\mathcal{Y}_{<t}$.

CMCDR matches the contribution profile of each old-class output to its frozen
counterpart. The replay-free objective is
\begin{equation}
\mathcal{L}^{\mathrm{free}}_{\mathrm{cmcdr}}
=
\frac{1}{N_t|\mathcal{Y}_{<t}|}
\sum_{\mathbf{x}_i\in\mathcal{X}_{\mathrm{new}}}
\sum_{k\in\mathcal{Y}_{<t}}
d_{\mathrm{cp}}
\left(
\mathbf{C}^{k}_{i,t},
\mathbf{C}^{k}_{i,t-1}
\right).
\label{eq:free_cmcdr_loss}
\end{equation}

\subsection{Overall Training Objective}

At incremental step $t$, CMCDR is integrated with the objective of the
underlying continual learning method as
\begin{equation}
\mathcal{L}_{t}
=
\mathcal{L}_{\mathrm{base}}
+
\lambda_{\mathrm{cmcdr}}
\mathcal{L}_{\mathrm{cmcdr}},
\label{eq:overall_objective}
\end{equation}
where $\mathcal{L}_{\mathrm{base}}$ denotes the original objective of the
underlying continual learning method, $\mathcal{L}_{\mathrm{cmcdr}}$ denotes
the corresponding replay-based or replay-free CMCDR objective, and
$\lambda_{\mathrm{cmcdr}}>0$ controls the regularization strength.

\paragraph{Theoretical analysis.}Proposition~\ref{prop:contribution_margin_loss}
  shows how changes in modality contributions can lead to forgetting.
Proposition~\ref{prop:representation_insufficient} establishes that
representation stability does not guarantee contribution stability.
\textbf{Appendix~\ref{app:cmcdr_theory}} provides
 theoretical guarantees that CMCDR mitigates MCD and thereby reduces forgetting.

\section{EXPERIMENTS}
We evaluate the proposed CMCDR on representative MMCL benchmarks by addressing three questions:
\begin{itemize}
    \item \textbf{Q1}. Does CMCDR improve the performance of existing MMCL methods?
    \item \textbf{Q2}. Does CMCDR provide complementary gains to representation-preserving MMCL methods?
    \item \textbf{Q3}. Does CMCDR reduce modality contribution drift and alleviate forgetting?
\end{itemize}

\subsection{EXPERIMENTAL SETUPS}

\textbf{Benchmarks.} We evaluate CMCDR under two representative MMCL protocols:
 class-incremental recognition and continual multimodal question answering. 
 For class-incremental learning, we use AVE~\citep{tian2018ave},
  Kinetics-Sounds~\citep{arandjelovic2017look}, and UESTC-MMEA-CL~\citep{xu2024uestc}.
   For continual QA, we evaluate on VQAv2~\citep{goyal2017vqav2} under the
    VQACL protocol~\citep{zhang2023vqacl}, and on Split-AVQA~\citep{yang2022avqa} 
    and Split-MUSIC-AVQA~\citep{li2022musicavqa} under the AVQACL 
    protocol~\citep{wu2025avqacl}. For VQACL and AVQACL, we instantiate 
    CMCDR by constructing contribution profiles from answer scores, with 
    implementation details provided in \textbf{Appendix~\ref{app:vqa_cmcdr}}.

\textbf{Baselines.} We compare CMCDR with representative continual learning and MMCL
 methods for audio-visual classification, continual visual question answering, and
  audio-visual question answering~\citep{li2017lwf,kirkpatrick2017ewc,rebuffi2017icarl,
  pian2023avcil,mo2023cign,yue2024mmal,zhang2023vqacl,qian2023dbi,wu2025avqacl}.
   We incorporate CMCDR into each baseline as an additional regularizer 
   for preserving modality contributions. CMCDR uses stored exemplars when 
   available. Otherwise, it uses the outputs of the frozen model on current-task probes.
    Apart from the additional CMCDR regularization term, all baselines retain their 
    original architectures, loss functions, and evaluation protocols.
     Implementation details are provided in \textbf{Appendix~\ref{app:main_training_setup}}.
      In addition, \textbf{Appendix~\ref{app:comparison_distinctiveness}} discusses the
       distinction between CMCDR and modality-balancing regularization methods.

\textbf{Evaluation.} We report the replay memory budget (\#Mem, measured as the total 
number of stored exemplars), average accuracy (Avg. Acc), and average forgetting (Avg. 
Forget). For contribution analysis, we additionally report the MCD score (MCD; lower is better).
 Values in parentheses indicate the changes obtained after adding CMCDR.

\begin{table}[htbp]
\centering
\scriptsize
\setlength{\tabcolsep}{4pt}
\caption{Plug-and-play results on multimodal class-incremental learning.}
\label{tab:class_incremental_results}
\resizebox{\textwidth}{!}{
\begin{tabular}{lccccccccc}
\toprule
\textbf{Method} &
\multicolumn{3}{c}{\textbf{AVE}} &
\multicolumn{3}{c}{\textbf{Kinetics-Sounds}} &
\multicolumn{3}{c}{\textbf{UESTC-MMEA-CL}} \\
\cmidrule(lr){2-4}\cmidrule(lr){5-7}\cmidrule(lr){8-10}
& \textbf{\#Mem} & \textbf{Avg. Acc. $\uparrow$} & \textbf{Avg. Forget. $\downarrow$}
& \textbf{\#Mem} & \textbf{Avg. Acc. $\uparrow$} & \textbf{Avg. Forget. $\downarrow$}
& \textbf{\#Mem} & \textbf{Avg. Acc. $\uparrow$} & \textbf{Avg. Forget. $\downarrow$} \\
\midrule
LwF & None & 56.61 & 35.11 & None & 65.54 & 16.55 & None & 17.13 & 49.37 \\
\rowcolor{cmcdrgreen}
\textbf{LwF + CMCDR} & None & \textbf{60.61} \accdelta{4.00} & \textbf{26.35} \downdelta{8.76} & None & \textbf{69.12} \accdelta{3.58} & \textbf{12.64} \downdelta{3.91} & None & \textbf{22.85} \accdelta{5.72} & \textbf{38.65} \downdelta{10.72} \\
EWC & None & 41.28 & 68.97 & None & 58.42 & 31.84 & None & 19.57 & 29.84 \\
\rowcolor{cmcdrgreen}
\textbf{EWC + CMCDR} & None & \textbf{47.36} \accdelta{6.08} & \textbf{57.42} \downdelta{11.55} & None & \textbf{62.73} \accdelta{4.31} & \textbf{24.96} \downdelta{6.88} & None & \textbf{24.78} \accdelta{5.21} & \textbf{22.06} \downdelta{7.78} \\
iCaRL & 340 & 64.20 & 25.60 & 500 & 65.54 & 40.57 & 320 & 71.83 & 22.86 \\
\rowcolor{cmcdrgreen}
\textbf{iCaRL + CMCDR} & 340 & \textbf{71.86} \accdelta{7.66} & \textbf{17.36} \downdelta{8.24} & 500 & \textbf{71.29} \accdelta{5.75} & \textbf{34.29} \downdelta{6.28} & 320 & \textbf{75.12} \accdelta{3.29} & \textbf{17.15} \downdelta{5.71} \\
AV-CIL & 340 & 71.37 & 6.77 & 500 & 73.06 & 6.48 & 320 & 80.27 & 18.74 \\
\rowcolor{cmcdrgreen}
\textbf{AV-CIL + CMCDR} & 340 & \textbf{75.86} \accdelta{4.49} & \textbf{5.31} \downdelta{1.46} & 500 & \textbf{76.14} \accdelta{3.08} & \textbf{4.55} \downdelta{1.93} & 320 & \textbf{83.68} \accdelta{3.41} & \textbf{13.36} \downdelta{5.38} \\
CIGN & 340 & 73.92 & 10.84 & 500 & 74.35 & 8.76 & 320 & 82.14 & 16.27 \\
\rowcolor{cmcdrgreen}
\textbf{CIGN + CMCDR} & 340 & \textbf{74.85} \accdelta{0.93} & \textbf{7.11} \downdelta{3.73} & 500 & \textbf{75.23} \accdelta{0.88} & \textbf{6.02} \downdelta{2.74} & 320 & \textbf{83.75} \accdelta{1.61} & \textbf{12.97} \downdelta{3.30} \\
MMAL & None & 70.28 & 19.64 & None & 71.46 & 13.72 & None & 75.57 & 24.83 \\
\rowcolor{cmcdrgreen}
\textbf{MMAL + CMCDR} & None & \textbf{72.63} \accdelta{2.35} & \textbf{15.02} \downdelta{4.62} & None & \textbf{74.16} \accdelta{2.70} & \textbf{9.86} \downdelta{3.86} & None & \textbf{78.12} \accdelta{2.55} & \textbf{20.94} \downdelta{3.89} \\
\midrule
Upper-bound & -- & 77.25 & -- & -- & 80.43 & -- & -- & 90.12 & -- \\
\bottomrule
\end{tabular}
}
\end{table}

\subsection{RESULTS}

\textbf{Plug-and-Play Effectiveness (Q1).} We evaluate the plug-and-play effectiveness of CMCDR across conventional
CL baselines and specialized MMCL methods using the integration procedures described
in \textbf{Appendix~\ref{app:experimental_details}}.
Tables~\ref{tab:class_incremental_results} and~\ref{tab:vqa_results} show consistent
improvements in two distinct MMCL settings: multimodal class-incremental learning
and continual multimodal question answering. Across both replay-based and replay-free
methods, CMCDR improves average accuracy by 0.88$\mathord{\sim}$7.66 points and reduces average
forgetting by 0.85$\mathord{\sim}$11.55 points. These improvements are consistently observed for
conventional CL baselines, including LwF, EWC, and iCaRL, as well as specialized
MMCL methods, including AV-CIL, CIGN, MMAL, VQACL, TRIPLET, and AVQACL.
\begin{table}[htbp]
\centering
\scriptsize
\setlength{\tabcolsep}{4pt}
\caption{Plug-and-play results on continual multimodal question answering.}
\label{tab:vqa_results}
\resizebox{\textwidth}{!}{
\begin{tabular}{lccccccccc}
\toprule
\textbf{Method} &
\multicolumn{3}{c}{\textbf{VQAv2}} &
\multicolumn{3}{c}{\textbf{Split-AVQA}} &
\multicolumn{3}{c}{\textbf{Split-MUSIC-AVQA}} \\
\cmidrule(lr){2-4}\cmidrule(lr){5-7}\cmidrule(lr){8-10}
& \textbf{\#Mem} & \textbf{Avg. Acc. $\uparrow$} & \textbf{Avg. Forget. $\downarrow$}
& \textbf{\#Mem} & \textbf{Avg. Acc. $\uparrow$} & \textbf{Avg. Forget. $\downarrow$}
& \textbf{\#Mem} & \textbf{Avg. Acc. $\uparrow$} & \textbf{Avg. Forget. $\downarrow$} \\
\midrule
LwF & None & 17.42 & 26.85 & None & 21.53 & 30.49 & None & 30.93 & 30.14 \\
\rowcolor{cmcdrgreen}
\textbf{LwF + CMCDR} & None & \textbf{22.35} \accdelta{4.93} & \textbf{21.01} \downdelta{5.84} & None & \textbf{25.81} \accdelta{4.28} & \textbf{24.11} \downdelta{6.38} & None & \textbf{35.14} \accdelta{4.21} & \textbf{24.88} \downdelta{5.26} \\
EWC & None & 15.77 & 30.62 & None & 18.89 & 40.69 & None & 28.76 & 32.58 \\
\rowcolor{cmcdrgreen}
\textbf{EWC + CMCDR} & None & \textbf{19.15} \accdelta{3.38} & \textbf{25.03} \downdelta{5.59} & None & \textbf{22.38} \accdelta{3.49} & \textbf{33.27} \downdelta{7.42} & None & \textbf{32.67} \accdelta{3.91} & \textbf{26.41} \downdelta{6.17} \\
iCaRL & 5000 & 30.62 & 12.48 & 5000 & 25.28 & 33.76 & 700 & 28.46 & 33.16 \\
\rowcolor{cmcdrgreen}
\textbf{iCaRL + CMCDR} & 5000 & \textbf{36.64} \accdelta{6.02} & \textbf{8.62} \downdelta{3.86} & 5000 & \textbf{29.31} \accdelta{4.03} & \textbf{27.45} \downdelta{6.31} & 700 & \textbf{32.33} \accdelta{3.87} & \textbf{27.12} \downdelta{6.04} \\
VQACL & 5000 & 38.77 & 3.96 & 5000 & 28.91 & 11.76 & 700 & 31.08 & 28.64 \\
\rowcolor{cmcdrgreen}
\textbf{VQACL + CMCDR} & 5000 & \textbf{41.36} \accdelta{2.59} & \textbf{2.71} \downdelta{1.25} & 5000 & \textbf{32.47} \accdelta{3.56} & \textbf{8.34} \downdelta{3.42} & 700 & \textbf{33.98} \accdelta{2.90} & \textbf{23.47} \downdelta{5.17} \\
TRIPLET & None & 34.96 & 7.24 & None & 26.72 & 16.83 & None & 31.47 & 23.51 \\
\rowcolor{cmcdrgreen}
\textbf{TRIPLET + CMCDR} & None & \textbf{38.21} \accdelta{3.25} & \textbf{4.86} \downdelta{2.38} & None & \textbf{29.86} \accdelta{3.14} & \textbf{14.26} \downdelta{2.57} & None & \textbf{34.59} \accdelta{3.12} & \textbf{20.07} \downdelta{3.44} \\
AVQACL & 5000 & 36.12 & 6.85 & 5000 & 32.05 & 2.47 & 700 & 33.64 & 27.08 \\
\rowcolor{cmcdrgreen}
\textbf{AVQACL + CMCDR} & 5000 & \textbf{39.44} \accdelta{3.32} & \textbf{4.31} \downdelta{2.54} & 5000 & \textbf{34.18} \accdelta{2.13} & \textbf{1.62} \downdelta{0.85} & 700 & \textbf{36.12} \accdelta{2.48} & \textbf{24.95} \downdelta{2.13} \\
\midrule
Upper-bound & -- & 51.64 & -- & -- & 55.42 & -- & -- & 65.57 & -- \\
\bottomrule
\end{tabular}
}
\end{table}

\noindent\textbf{Complementarity with Representation Preservation (Q2).}
We evaluate whether CMCDR complements representation-preservation methods. For both
LwF and iCaRL, we compare the baseline, each regularizer alone, and their
combinations under the same training protocol.
\begin{figure*}[htbp]
    \centering
    \includegraphics[width=\textwidth]{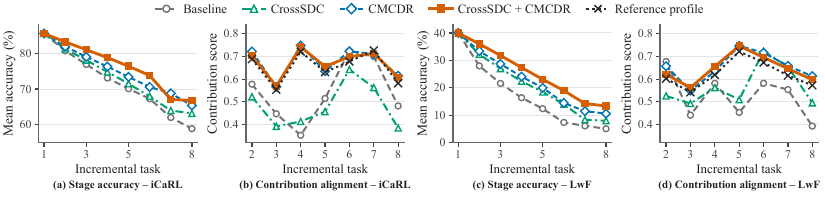}
    \vspace{-2.0em}
    \caption{Complementarity of representation and contribution preservation on UESTC-MMEA-CL.
Panels (a,b) use replay-based iCaRL, while (c,d) use replay-free LwF.
Panels (b,d) show the visual-modality contribution score for each task.}
    \label{fig:rep_preservation_complementarity}
    \vspace{-0.6em}
\end{figure*}
\begin{table}[htbp]
\centering
\scriptsize
\setlength{\tabcolsep}{4pt}
\caption{Complementarity with representation-preservation methods built on
continual-learning baselines.}
\label{tab:rep_preservation_complementarity}
\resizebox{\textwidth}{!}{
\begin{tabular}{lcccccc}
\toprule
\textbf{Method} &
\multicolumn{3}{c}{\textbf{AVE}} &
\multicolumn{3}{c}{\textbf{UESTC-MMEA-CL}} \\
\cmidrule(lr){2-4}\cmidrule(lr){5-7}
& \textbf{Avg. Acc. $\uparrow$} & \textbf{Avg. Forget. $\downarrow$} & \textbf{MCD $\downarrow$}
& \textbf{Avg. Acc. $\uparrow$} & \textbf{Avg. Forget. $\downarrow$} & \textbf{MCD $\downarrow$} \\
\midrule
LwF & 56.61 & 35.11 & 0.842 & 17.13 & 49.37 & 0.513 \\
\rowcolor{cmcdrgreen}
\textbf{LwF + CMCDR} & \textbf{60.61} \accdelta{4.00} & \textbf{26.35} \downdelta{8.76} & \textbf{0.401} \downdelta{0.441} & \textbf{22.85} \accdelta{5.72} & \textbf{38.65} \downdelta{10.72} & \textbf{0.299} \downdelta{0.214} \\
LwF + RepAlign & 61.84 & 29.06 & 0.674 & 20.59 & 43.22 & 0.450 \\
\rowcolor{cmcdrgreen}
\textbf{LwF + RepAlign + CMCDR} & \textbf{65.22} \accdelta{3.38} & \textbf{23.41} \downdelta{5.65} & \textbf{0.251} \downdelta{0.423} & \textbf{24.88} \accdelta{4.29} & \textbf{35.13} \downdelta{8.09} & \textbf{0.286} \downdelta{0.164} \\
LwF + CrossSDC & 60.73 & 30.18 & 0.641 & 21.31 & 41.79 & 0.436 \\
\rowcolor{cmcdrgreen}
\textbf{LwF + CrossSDC + CMCDR} & \textbf{64.04} \accdelta{3.31} & \textbf{24.72} \downdelta{5.46} & \textbf{0.278} \downdelta{0.363} & \textbf{25.62} \accdelta{4.31} & \textbf{33.87} \downdelta{7.92} & \textbf{0.243} \downdelta{0.193} \\
\midrule
iCaRL & 64.20 & 25.60 & 0.779 & 71.83 & 22.86 & 0.464 \\
\rowcolor{cmcdrgreen}
\textbf{iCaRL + CMCDR} & \textbf{71.86} \accdelta{7.66} & \textbf{17.36} \downdelta{8.24} & \textbf{0.389} \downdelta{0.390} & \textbf{75.12} \accdelta{3.29} & \textbf{17.15} \downdelta{5.71} & \textbf{0.288} \downdelta{0.176} \\
iCaRL + RepAlign & 70.30 & 19.20 & 0.625 & 73.01 & 20.12 & 0.431 \\
\rowcolor{cmcdrgreen}
\textbf{iCaRL + RepAlign + CMCDR} & \textbf{72.65} \accdelta{2.35} & \textbf{16.39} \downdelta{2.81} & \textbf{0.203} \downdelta{0.422} & \textbf{76.96} \accdelta{3.95} & \textbf{16.77} \downdelta{3.35} & \textbf{0.235} \downdelta{0.196} \\
iCaRL + CrossSDC & 67.80 & 22.50 & 0.699 & 73.29 & 19.21 & 0.419 \\
\rowcolor{cmcdrgreen}
\textbf{iCaRL + CrossSDC + CMCDR} & \textbf{72.89} \accdelta{5.09} & \textbf{19.96} \downdelta{2.54} & \textbf{0.214} \downdelta{0.485} & \textbf{76.69} \accdelta{3.40} & \textbf{15.38} \downdelta{3.83} & \textbf{0.186} \downdelta{0.233} \\
\bottomrule
\end{tabular}
}
\end{table}

As summarized in Table~\ref{tab:rep_preservation_complementarity} and
Figure~\ref{fig:rep_preservation_complementarity}, adding CMCDR to RepAlign or
CrossSDC improves accuracy by 2.35$\mathord{\sim}$5.09 points, reduces forgetting by
2.54$\mathord{\sim}$8.09 points, and lowers MCD scores by 0.164$\mathord{\sim}$0.485. Compared with CrossSDC alone,
the CrossSDC + CMCDR variants achieve higher accuracy in
Figure~\ref{fig:rep_preservation_complementarity}(a,c) and track the reference
contribution profile more closely in
Figure~\ref{fig:rep_preservation_complementarity}(b,d). These results demonstrate
that CMCDR complements representation preservation by directly constraining
residual contribution drift.

\noindent\textbf{MCD and Forgetting Analysis (Q3).} To examine the relationship between contribution drift and forgetting across
old classes, we compare classes with different baseline MCD scores.
Figure~\ref{fig:high_low_mcd_classes}(a) ranks old classes according to their
baseline MCD scores. CMCDR consistently reduces MCD across classes,
with larger reductions for classes with higher baseline drift.
Figure~\ref{fig:high_low_mcd_classes}(b) further groups old classes into low- and
high-MCD groups and compares their accuracy across incremental tasks.
Under the baseline, high-MCD classes exhibit larger accuracy decreases, whereas CMCDR
reduces contribution drift, improves their classification accuracy, and
narrows the performance gap between the two groups.
\textbf{Appendix~\ref{app:mcd_forgetting}} further analyzes the paired changes
in MCD and class-wise forgetting on AVE under iCaRL and LwF.

\begin{figure}[htbp]
    \centering
    \includegraphics[width=0.94\textwidth]{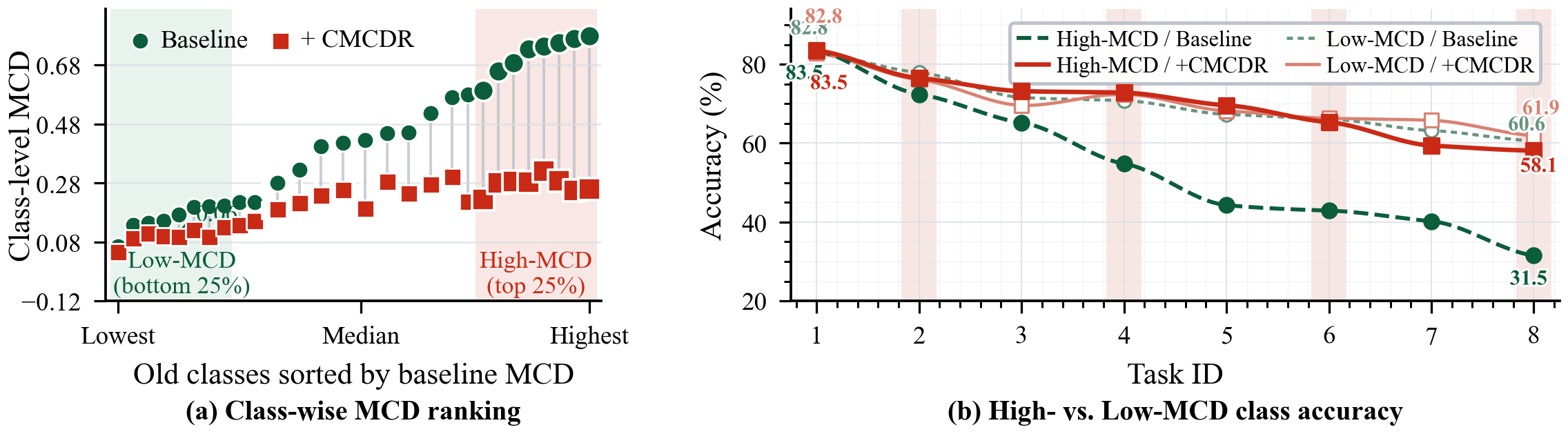}
    \caption{Class-wise effect of CMCDR on UESTC-MMEA-CL with replay-based iCaRL.}
    \label{fig:high_low_mcd_classes}
\end{figure}
\vspace{-0.4em}

\section{Limitations and Future Work}

\noindent\textbf{Task Similarity.}
CMCDR constrains deviations from previous contribution profiles and may therefore
restrict contribution changes that benefit old-task predictions. When tasks share
similar modality contribution patterns, changes induced by learning new tasks may
also improve old-task predictions, as discussed in
\textbf{Appendix~\ref{app:mcd_forgetting}}. Future work will investigate how
similarity in modality contribution patterns affects contribution shifts and develop
adaptive regularization that preserves old-task contributions while allowing
beneficial changes.

\noindent\textbf{Computational Cost.}
Computing exact contribution profiles with CMCDR requires
$2^{|\mathcal M|}$ modality-subset evaluations. The computational cost therefore
increases exponentially with the number of modalities
(\textbf{Appendix~\ref{app:efficiency_scalability}}).
Future work will explore sampled or adaptive coalition selection to
estimate contribution profiles more efficiently.

\section{Conclusion}

In this work, we study modality contribution drift, an underexplored problem in MMCL.
An intervention-based contribution estimator and an MCD score are introduced to quantify changes in modality contributions across incremental stages.
Our theoretical and empirical analyses show how contribution drift can
lead to forgetting and further show that existing MMCL and conventional CL
methods do not effectively mitigate MCD. To address this problem,
we propose CMCDR to preserve old-task contribution profiles in both
replay-based and replay-free settings. CMCDR aligns contribution profiles
between the current and frozen previous models using stored old samples
or current-task samples as probes for old-class outputs, respectively.
Experiments on six benchmarks spanning multimodal class-incremental learning and continual
multimodal QA show that CMCDR consistently reduces MCD, improves average
accuracy, and reduces forgetting. These results demonstrate that CMCDR effectively complements representation preservation in MMCL.

\section*{AI Use Statement}

We used generative AI tools (ChatGPT and Codex) to help improve the
readability and grammar of parts of the manuscript, draft and edit
implementation code, and design and refine the experimental protocol,
including baselines and ablation studies.

All AI-assisted text and code were carefully reviewed, revised, and verified by
the authors. The authors take full responsibility for the final content of this
submission, including all text, claims, code, and experimental results.

\section*{Reproducibility Statement}

\hyperref[sec:contribution_drift]{Section~\ref*{sec:contribution_drift}}
and \hyperref[sec:method]{Section~\ref*{sec:method}}
specify the contribution estimator, MCD scores, and CMCDR objectives.
\hyperref[app:theory_proof]{Appendix~\ref*{app:theory_proof}}
provides the theoretical proofs, and
\hyperref[app:experimental_details]{Appendix~\ref*{app:experimental_details}}
describes task construction, modality-removal rules, baseline implementations,
and training settings.

\bibliography{references}
\bibliographystyle{references}

\newpage

\appendix
\raggedbottom

\pdfbookmark[1]{Appendix: Table of Contents}{appendix-toc}
\section*{Appendix: Table of Contents}
\vspace{0.4em}
\begingroup
\normalsize
\setlength{\parindent}{0pt}
\setlength{\parskip}{0pt}
\newcommand{\appoverviewsection}[2]{%
  \par\addvspace{0.45em}%
  \noindent
  \makebox[2.7em][l]{\hyperref[#1]{\textbf{\ref*{#1}}}}%
  \hyperref[#1]{\textbf{#2}}%
  \nobreak\hspace{0.35em}%
  \leaders\hbox to 0.55em{\hss.\hss}\hfill
  \nobreak\makebox[2em][r]{\hyperref[#1]{\pageref*{#1}}}\par}
\newcommand{\appoverviewsubsection}[2]{%
  \par\addvspace{0.05em}%
  \noindent\hspace*{2.7em}%
  \makebox[3.2em][l]{\hyperref[#1]{\ref*{#1}}}%
  \hyperref[#1]{#2}%
  \nobreak\hspace{0.35em}%
  \leaders\hbox to 0.55em{\hss.\hss}\hfill
  \nobreak\makebox[2em][r]{\hyperref[#1]{\pageref*{#1}}}\par}

\appoverviewsection{app:theory_proof}{Theoretical Analysis and Proofs}
\appoverviewsubsection{app:estimator_analysis}{Interventional Contribution Estimator}
\appoverviewsubsection{app:contribution_forgetting_proof}{Contribution Drift and Forgetting}
\appoverviewsubsection{app:later_class_competition}{Fixed-Reference Drift and Later-Class Competition}
\appoverviewsubsection{app:representation_insufficient_proof}{Why Representation Stability Is Insufficient ?}
\appoverviewsubsection{app:cmcdr_theory}{Contribution-Drift Bounds for CMCDR}
\appoverviewsubsection{app:interaction_regularization}{M\"obius Matching vs. Centered-Response Matching}

\appoverviewsection{app:experimental_details}{Experimental Settings and Implementation Details}
\appoverviewsubsection{app:experimental_settings}{Experimental Settings for Exploratory Analyses}
\appoverviewsubsection{app:main_training_setup}{Main Training and Evaluation Setup}
\appoverviewsubsection{app:vqa_cmcdr}{CMCDR for Continual Multimodal Question Answering}
\appoverviewsubsection{app:modality_removal}{Modality-Removal Realizations}

\appoverviewsection{app:additional_experiments}{Additional Experiments}
\appoverviewsubsection{app:seed_robustness}{Random-Seed Robustness}
\appoverviewsubsection{app:mcd_forgetting}{Modality Contribution Drift and Class-wise Forgetting}
\appoverviewsubsection{app:efficiency_scalability}{Computational Efficiency and Scalability}
\appoverviewsubsection{app:hyperparameter_analysis}{Hyperparameter Sensitivity Analysis}
\appoverviewsubsection{app:probe_transfer}{Replay-Free Probe Transfer}
\appoverviewsubsection{app:contribution_ablation}{Ablation and Controlled Comparisons}
\appoverviewsubsection{app:alternative_contribution_estimators}{Robustness to Alternative Contribution Estimators}

\appoverviewsection{app:comparison_distinctiveness}{Comparison with Related Methods}
\appoverviewsubsection{app:mmcl_relationship}{Multimodal Continual Learning}
\appoverviewsubsection{app:balanced_multimodal_relationship}{Balanced Multimodal Learning}
\appoverviewsubsection{app:modality_balancing_comparison}{Comparison with Modality-Balancing Methods}
\endgroup

\section{Theoretical Analysis and Proofs}
\label{app:theory_proof}

\subsection{Interventional Contribution Estimator}
\label{app:estimator_analysis}

For a sample and model, we use the scalar benefit in
Eq.~\ref{eq:benefit_function} to compare retaining a modality subset $\mathcal S$
with the no-modality baseline $\emptyset$. For logits restricted to $\mathcal A$, define
\begin{equation}
p_{i,a}^{\mathcal{A}}(\mathcal{S})
=
\frac{\exp(z_{i,a}(\mathcal{S}))}
{\sum_{\ell\in\mathcal{A}}\exp(z_{i,\ell}(\mathcal{S}))}.
\end{equation}
Here $z_{i,\ell}(\mathcal S)$ denotes the logit for candidate class $\ell\in\mathcal A$.
Substituting the softmax probability into Eq.~\ref{eq:benefit_function} gives
\begin{equation}
B_i^{a,\mathcal{A}}(\mathcal{S})
=\log\frac{p_{i,a}^{\mathcal{A}}(\mathcal{S})}
{1-p_{i,a}^{\mathcal{A}}(\mathcal{S})}.
\end{equation}
Thus $B_i^{a,\mathcal{A}}(\mathcal{S})$ is the one-vs-rest log-odds of
target class $a$. The contrast
$B_i^{a,\mathcal{A}}(\mathcal{S})-B_i^{a,\mathcal{A}}(\emptyset)$
is the log-odds ratio relative to the no-modality baseline, quantifying the
total decision-level effect of retaining $\mathcal{S}$. It is unchanged if all logits in $\mathcal{A}$ are shifted by the
same constant, and it avoids the saturation of probabilities by working in
log-odds space.
Let $\gamma_i^{a,\mathcal A}(\mathcal S)=z_{i,a}(\mathcal S)-\max_{\ell\in\mathcal A,\,\ell\neq a}z_{i,\ell}(\mathcal S)$
denote the target-class logit margin. Then
\begin{equation}
\gamma_i^{a,\mathcal{A}}(\mathcal{S})-\log(|\mathcal{A}|-1)
\leq
B_i^{a,\mathcal{A}}(\mathcal{S})
\leq
\gamma_i^{a,\mathcal{A}}(\mathcal{S}).
\label{eq:benefit_margin_sandwich}
\end{equation}
The benefit is therefore a smooth multiclass margin that compares the target
logit with the log-sum-exp of all competing logits in $\mathcal A$.

Each
$B_i^{a,\mathcal{A}}(\mathcal{S})$ is evaluated by one controlled forward pass
as in Eq.~\ref{eq:modality_intervention}, using the modality-removal rule in
Appendix~\ref{app:modality_removal}.

Finally, we use the M\"obius coefficients in Eq.~\ref{eq:general_contribution}
to separate singleton effects relative to the no-modality baseline from
higher-order residual effects not explained by lower-order subsets.
For example, in the two-modality case,
\begin{equation}
C_i^{a,\mathcal{A}}(\{m,n\})
=
B_i^{a,\mathcal{A}}(\{m,n\})
-B_i^{a,\mathcal{A}}(\{m\})
-B_i^{a,\mathcal{A}}(\{n\})
+B_i^{a,\mathcal{A}}(\emptyset),
\end{equation}
which isolates the interaction beyond the two singleton effects. By M\"obius
inversion,
\begin{equation}
B_i^{a,\mathcal{A}}(\mathcal{S})
=
\sum_{\mathcal{T}\subseteq\mathcal{S}}
C_i^{a,\mathcal{A}}(\mathcal{T}),
\qquad
C_i^{a,\mathcal{A}}(\emptyset)=B_i^{a,\mathcal{A}}(\emptyset).
\end{equation}
Setting $\mathcal{S}=\mathcal{M}$ and subtracting the empty-set term yields
the completeness identity in Eq.~\ref{eq:contribution_completeness}.

\subsection{Contribution Drift and Forgetting}
\label{app:contribution_forgetting_proof}

\begin{restatedproposition}
\textbf{Proposition~\ref{prop:contribution_margin_loss}.}
\textit{Contribution drift can cross the retention margin.}
For an old sample $i$, let $\tau(y_i)\leq r<t$, fix a candidate set
$\mathcal A_i\subseteq\mathcal Y_{\leq r}$ containing $y_i$ with
$|\mathcal A_i|\geq2$, and use the same
modality-removal rule at stages $r$ and $t$. All benefits and contribution
profiles use target $y_i$ and candidate set $\mathcal A_i$.
Define $\bar B_{i,s}(\mathcal S)=B_{i,s}(\mathcal S)-B_{i,s}(\emptyset)$ and
\[
m^{\mathrm{keep}}_{i,r\to t}
=B_{i,t}(\emptyset)+\bar B_{i,r}(\mathcal M).
\]
Let $\mathbf c_{i,s}$ denote the contribution vector at stage
$s\in\{r,t\}$ under $(y_i,\mathcal A_i)$.
Let $\gamma_{i,t}$ denote the target-class logit margin over its strongest
competitor in $\mathcal A_i$ at stage $t$. Then
\begin{equation*}
\begin{aligned}
B_{i,t}(\mathcal M)
&=m^{\mathrm{keep}}_{i,r\to t}
+\mathbf 1^\top(\mathbf c_{i,t}-\mathbf c_{i,r}),\\
\gamma_{i,t}
&\geq m^{\mathrm{keep}}_{i,r\to t}
-\|\mathbf c_{i,t}-\mathbf c_{i,r}\|_1.
\end{aligned}
\end{equation*}
If $\|\mathbf c_{i,t}-\mathbf c_{i,r}\|_1<m^{\mathrm{keep}}_{i,r\to t}$, then $y_i$
has a strictly higher logit than every other class in $\mathcal A_i$.
Conversely, if $m^{\mathrm{keep}}_{i,r\to t}>0$ and $\gamma_{i,t}\leq0$,
then $\|\mathbf c_{i,t}-\mathbf c_{i,r}\|_1\geq m^{\mathrm{keep}}_{i,r\to t}$.
\end{restatedproposition}

\begin{proof}
Let $\mathbf c_s$ be the contribution vector at stage $s\in\{r,t\}$ under
$(y_i,\mathcal A_i)$. By M\"obius inversion,
\begin{equation}
B_{i,t}(\mathcal M)
=m^{\mathrm{keep}}_{i,r\to t}
+\mathbf 1^\top(\mathbf c_t-\mathbf c_r).
\label{eq:appendix_contribution_margin_identity}
\end{equation}
Equation~\ref{eq:benefit_margin_sandwich} and the triangle inequality give
\begin{equation}
\gamma_{i,t}\geq B_{i,t}(\mathcal M)
\geq m^{\mathrm{keep}}_{i,r\to t}-\|\mathbf c_t-\mathbf c_r\|_1.
\label{eq:appendix_contribution_margin_sandwich}
\end{equation}
If $\|\mathbf c_t-\mathbf c_r\|_1<m^{\mathrm{keep}}_{i,r\to t}$, then
$\gamma_{i,t}>0$. Conversely, $\gamma_{i,t}\leq0$ implies
$\|\mathbf c_t-\mathbf c_r\|_1\geq m^{\mathrm{keep}}_{i,r\to t}$.
\end{proof}

The retained margin equals the reference full-input benefit plus the change
in the no-modality response:
\[
m^{\mathrm{keep}}_{i,r\to t}
=B_{i,r}(\mathcal M)
+B_{i,t}(\emptyset)-B_{i,r}(\emptyset).
\]

A prediction can flip even when the modality-subset representations remain
unchanged. Consider two modalities
$\mathcal M=\{1,2\}$ and two classes $\mathcal A=\{a,b\}$. Let $x=(1,1)$
and use zero-valued modality removal. For this fixed sample, we omit the
sample index and keep the subset representations unchanged across stages:
\[
u_r(\mathcal S)=u_t(\mathcal S)=x^{\mathcal S},\qquad
z_{s,a}(x^{\mathcal S})=w_s^\top x^{\mathcal S},\qquad
z_{s,b}(x^{\mathcal S})=0.
\]
For any $\gamma,\eta>0$, choose
\[
w_r=\left(\frac{\gamma}{2},\frac{\gamma}{2}\right),\qquad
w_t=\left(\frac{\gamma}{2},-\frac{\gamma}{2}-\eta\right).
\]
Since $\mathcal A$ is binary, the benefit equals the target-class logit
margin. The resulting benefits and contribution profiles are
\[
\begin{aligned}
&B_r(\emptyset)=B_t(\emptyset)=0,\qquad
B_r(\mathcal M)=\gamma,\qquad
B_t(\mathcal M)=-\eta,\\
&\mathbf c_r=
\left(\frac{\gamma}{2},\frac{\gamma}{2},0\right),\qquad
\mathbf c_t=
\left(\frac{\gamma}{2},-\frac{\gamma}{2}-\eta,0\right),
\end{aligned}
\]
where
$\mathbf c_s=(C_s(\{1\}),C_s(\{2\}),C_s(\{1,2\}))$.
Here $m^{\mathrm{keep}}_{r\to t}=\gamma$.
The binary margin changes from $\gamma>0$ to $-\eta<0$, turning a correct
prediction into an incorrect one while preserving the modality-subset
representations and the no-modality response.

\subsection{Fixed-Reference Drift and Later-Class Competition}
\label{app:later_class_competition}

For an old sample $i$ with target $a=y_i$, let $\mathcal A_i$ denote its candidate set and
$\mathcal N_t=\mathcal Y_{\leq t}\setminus\mathcal A_i$ denote the later
classes. Define
\[
Z_{\mathcal A_i,t}(\mathcal S)
=
\sum_{\ell\in\mathcal A_i\setminus\{a\}}
\exp z_{i,t,\ell}(\mathcal S),
\qquad
Z_{\mathcal N_t,t}(\mathcal S)
=
\sum_{\ell\in\mathcal N_t}
\exp z_{i,t,\ell}(\mathcal S).
\]
For the centered benefit
$\bar B(\mathcal S)=B(\mathcal S)-B(\emptyset)$,
\begin{equation}
\bar B_{i,t}^{a,\mathcal A_i\cup\mathcal N_t}(\mathcal S)
=
\bar B_{i,t}^{a,\mathcal A_i}(\mathcal S)
-
\left[
q_{i,t}(\mathcal S)-q_{i,t}(\emptyset)
\right],
\qquad
q_{i,t}(\mathcal S)
=
\log\left(
1+
\frac{Z_{\mathcal N_t,t}(\mathcal S)}
     {Z_{\mathcal A_i,t}(\mathcal S)}
\right).
\label{eq:later_class_competition_decomposition}
\end{equation}
Thus, later classes change the coalition values whenever their relative
logit mass varies across modality subsets. A stage-dependent candidate set
would combine contribution changes within $\mathcal A_i$ with direct
competition from newly introduced classes.

MCD measures cumulative contribution drift from class introduction under the candidate set $\mathcal A_y=\mathcal Y_{\leq\tau(y)}$.
Fixing $\mathcal A_y$ excludes changes caused solely by candidate-set expansion
and keeps the competing classes unchanged across stages and methods.

In class-incremental learning, the model must continue to distinguish previously
learned classes as new classes arrive. At update $t$, CMCDR therefore matches
contributions under $\mathcal Y_{<t}$ to cover all previously learned classes,
including those introduced after $y$.
This set expands across updates but is shared by both models and all
modality-subset interventions within each update.

Eq.~\ref{eq:later_class_competition_decomposition} also yields a sufficient
condition for prediction retention over the full class space.
With $m^{\mathrm{keep}}_{i,r\to t}$ and the contribution vectors
computed under the same fixed
$\mathcal A_i$ as in \textbf{Proposition~\ref{prop:contribution_margin_loss}},
Eq.~\ref{eq:appendix_contribution_margin_sandwich} gives
\begin{equation}
B_{i,t}^{y_i,\mathcal A_i\cup\mathcal N_t}(\mathcal M)
\geq m^{\mathrm{keep}}_{i,r\to t}
-\|\mathbf c_{i,t}-\mathbf c_{i,r}\|_1-q_{i,t}(\mathcal M).
\label{eq:full_space_retention_lower_bound}
\end{equation}
Thus, $\|\mathbf c_{i,t}-\mathbf c_{i,r}\|_1+q_{i,t}(\mathcal M)<m^{\mathrm{keep}}_{i,r\to t}$
guarantees that, at stage $t$, $y_i$ has a strictly higher logit than
every other class in $\mathcal Y_{\leq t}$.

\subsection{Why Representation Stability Is Insufficient ?}
\label{app:representation_insufficient_proof}
\label{sub:proof_of_prop_representation_insufficient}

\begin{restatedproposition}
\textbf{Proposition~\ref{prop:representation_insufficient}.}
\textit{Representation stability does not guarantee contribution stability.}
For a sample $i$ with target old class $a$ and stages $\tau(a)\leq r<t$, fix
$\mathcal{A}=\mathcal{A}_a$. Let $\mathbf u_{i,s}(\mathcal S)$ denote the
modality-subset representation at stage $s$ in a common normed space, and let
$g_s^{a,\mathcal A}$ denote the corresponding target-vs-rest margin map, such that
$B_{i,s}^{a,\mathcal A}(\mathcal S)
=g_s^{a,\mathcal A}(\mathbf u_{i,s}(\mathcal S))$. Define
\begin{equation*}
\Delta_u(\mathcal S)
=
\|\mathbf u_{i,t}(\mathcal S)-\mathbf u_{i,r}(\mathcal S)\|,
\qquad
\Delta_g(\mathcal S)
=
\left|
g_t^{a,\mathcal A}(\mathbf u_{i,r}(\mathcal S))
-
g_r^{a,\mathcal A}(\mathbf u_{i,r}(\mathcal S))
\right|.
\end{equation*}
If $g_t^{a,\mathcal A}$ is $L_g$-Lipschitz with respect to this norm, then for
any nonempty $\mathcal T\subseteq\mathcal M$,
\begin{equation*}
\left|
C_{i,t}^{a,\mathcal A}(\mathcal T)
-
C_{i,r}^{a,\mathcal A}(\mathcal T)
\right|
\leq
\sum_{\mathcal S\subseteq\mathcal T}
\left(
L_g\Delta_u(\mathcal S)+\Delta_g(\mathcal S)
\right).
\end{equation*}
Let
$\bar B_{i,s}^{a,\mathcal A}(\mathcal S)
=
B_{i,s}^{a,\mathcal A}(\mathcal S)
-
B_{i,s}^{a,\mathcal A}(\emptyset)$.
If every nonempty $\mathcal S\subseteq\mathcal T$ satisfies
$|\bar B_{i,t}^{a,\mathcal A}(\mathcal S)
-\bar B_{i,r}^{a,\mathcal A}(\mathcal S)|
\leq\epsilon_{\bar B}$, then
$|C_{i,t}^{a,\mathcal A}(\mathcal T)
-C_{i,r}^{a,\mathcal A}(\mathcal T)|
\leq(2^{|\mathcal T|}-1)\epsilon_{\bar B}$.
Moreover, there exist identical modality-subset representations at stages $r$ and $t$
and two decision maps such that $\Delta_u(\mathcal S)=0$ for every $\mathcal S$,
while $C_{i,t}^{a,\mathcal A}(\mathcal T)
\neq C_{i,r}^{a,\mathcal A}(\mathcal T)$ for some nonempty $\mathcal T$.
\end{restatedproposition}

\begin{proof}
Fix the quantities in Proposition~\ref{prop:representation_insufficient}.
For any $\mathcal{S}\subseteq\mathcal{T}$, adding and subtracting
$g_t^{a,\mathcal A}(\mathbf u_{i,r}(\mathcal S))$ and applying the triangle
inequality and Lipschitz continuity gives
\begin{equation}
\left|
B_{i,t}^{a,\mathcal{A}}(\mathcal{S})
-
B_{i,r}^{a,\mathcal{A}}(\mathcal{S})
\right|
\leq
L_g\Delta_u(\mathcal{S})+\Delta_g(\mathcal{S}).
\end{equation}
Applying Eq.~\ref{eq:general_contribution} and the triangle inequality,
\begin{equation}
\left|C_{i,t}^{a,\mathcal{A}}(\mathcal{T})-C_{i,r}^{a,\mathcal{A}}(\mathcal{T})\right|
\leq \sum_{\mathcal{S}\subseteq\mathcal{T}}
\left|B_{i,t}^{a,\mathcal{A}}(\mathcal{S})-B_{i,r}^{a,\mathcal{A}}(\mathcal{S})\right|
\leq \sum_{\mathcal{S}\subseteq\mathcal{T}}
\left(L_g\Delta_u(\mathcal{S})+\Delta_g(\mathcal{S})\right),
\end{equation}
which proves Eq.~\ref{eq:representation_to_contribution_bound}.

For any nonempty
coalition $\mathcal{T}$, Eq.~\ref{eq:general_contribution} gives
\begin{equation}
C_{i,s}^{a,\mathcal{A}}(\mathcal{T})
=
\sum_{\mathcal{S}\subseteq\mathcal{T}}
(-1)^{|\mathcal{T}|-|\mathcal{S}|}
\bar{B}_{i,s}^{a,\mathcal{A}}(\mathcal{S}),
\label{eq:centered_mobius_appendix}
\end{equation}
because
$\sum_{\mathcal{S}\subseteq\mathcal{T}}
(-1)^{|\mathcal{T}|-|\mathcal{S}|}=(1-1)^{|\mathcal{T}|}=0$.
Under the centered-response condition in Proposition~\ref{prop:representation_insufficient},
Eq.~\ref{eq:centered_mobius_appendix} and the triangle inequality give
\begin{equation}
\left|C_{i,t}^{a,\mathcal{A}}(\mathcal{T})-C_{i,r}^{a,\mathcal{A}}(\mathcal{T})\right|
\leq \sum_{\emptyset\neq\mathcal{S}\subseteq\mathcal{T}}
\left|\bar{B}_{i,t}^{a,\mathcal{A}}(\mathcal{S})-\bar{B}_{i,r}^{a,\mathcal{A}}(\mathcal{S})\right|
\leq (2^{|\mathcal{T}|}-1)\epsilon_{\bar B}.
\label{eq:benefit_to_contribution_bound}
\end{equation}

To show that representation stability need not imply contribution stability,
consider two modalities and a binary linear classifier operating on fixed
post-fusion features.
Suppressing the fixed superscript $(a,\mathcal A)$, set
$B_{i,s}(\mathcal S)=w_s^\top\mathbf u_i(\mathcal S)$, with
$\mathbf u_{i,t}(\mathcal S)=\mathbf u_{i,r}(\mathcal S)=\mathbf u_i(\mathcal S)$
for every $\mathcal S$. Set
$\mathbf{u}_i(\emptyset)=0$, $\mathbf{u}_i(\{1\})=(1,0)^\top$,
$\mathbf{u}_i(\{2\})=(0,1)^\top$, and
$\mathbf{u}_i(\{1,2\})=(1,1)^\top$. Choose $w_r=(0,0)^\top$ and
$w_t=(1,-1)^\top$. Then $C_{i,r}(\{1\})=C_{i,r}(\{2\})=0$,
but $C_{i,t}(\{1\})=1$ and $C_{i,t}(\{2\})=-1$, while
$\bar B_{i,r}(\{1,2\})=\bar B_{i,t}(\{1,2\})=0$.
Thus, representation stability does not imply contribution stability, even
when the full-modality centered response is unchanged.
\end{proof}

\subsection{Contribution-Drift Bounds for CMCDR}
\label{app:cmcdr_theory}

\paragraph{Loss and contribution drift.}
Fix a sample $i$, a target class $a\in\mathcal A$ with $|\mathcal A|\geq2$,
and stages $r<t$. Use the same candidate set, modality-removal rule, and
coalition order at both stages. Let $D=2^{|\mathcal M|}-1$ and $\epsilon>0$. Write
\begin{equation}
\mathbf c_s=\left[C_{i,s}^{a,\mathcal A}(\mathcal T)\right]_{\emptyset\neq\mathcal T\subseteq\mathcal M},
\qquad m_s=\|\mathbf c_s\|_1,
\qquad \widehat{\mathbf c}_s=\frac{\mathbf c_s}{m_s+\epsilon},
\quad s\in\{r,t\}.
\label{eq:cmcdr_theory_profiles}
\end{equation}
The raw and normalized differences are
$d^{\mathrm{abs}}_{i,r\to t}=\|\mathbf c_t-\mathbf c_r\|_1$ and
$d^{\mathrm{rel}}_{i,r\to t}=\|\widehat{\mathbf c}_t-\widehat{\mathbf c}_r\|_1$.
The mean Smooth-$L_1$ loss in Eq.~\ref{eq:cmcdr_profile_distance} is
\begin{equation}
\rho_\delta(\mathbf u,\mathbf v)
=\frac{1}{D}\sum_{j=1}^{D}h_\delta(u_j-v_j),
\qquad
h_\delta(x)=
\begin{cases}
x^2/(2\delta), & |x|\leq\delta,\\
|x|-\delta/2, & |x|>\delta,
\end{cases}
\label{eq:cmcdr_smooth_l1_definition}
\end{equation}
where $\delta>0$. The per-profile loss is
\begin{equation}
\ell_{i,r\to t}
=\rho_\delta(\mathbf c_t,\mathbf c_r)
+\beta\rho_\delta(\widehat{\mathbf c}_t,\widehat{\mathbf c}_r),
\qquad \beta>0.
\label{eq:cmcdr_per_profile_loss}
\end{equation}
Each training comparison uses $r=t-1$ and $\mathcal A=\mathcal Y_{<t}$.

\begin{lemma}[Loss bounds contribution drift]
\label{lem:cmcdr_controls_mcd}
Let $\Psi_\delta(u)=\sqrt{2\delta u}+u$ for $u\geq0$. Then
\begin{equation}
d^{\mathrm{abs}}_{i,r\to t}\leq D\Psi_\delta(\ell_{i,r\to t}),
\qquad
d^{\mathrm{rel}}_{i,r\to t}\leq D\Psi_\delta(\ell_{i,r\to t}/\beta).
\label{eq:cmcdr_loss_to_mcd_bound}
\end{equation}
For any probability weighting $\pi$ over finitely many profile comparisons
with the same $D$, $\delta$, and $\beta$, let
$\bar\ell=\mathbb E_\pi[\ell_{i,r\to t}]$. Then
\begin{equation}
\mathbb E_\pi[d^{\mathrm{abs}}_{i,r\to t}]\leq D\Psi_\delta(\bar\ell),
\qquad
\mathbb E_\pi[d^{\mathrm{rel}}_{i,r\to t}]\leq D\Psi_\delta(\bar\ell/\beta).
\label{eq:cmcdr_average_loss_to_mcd_bound}
\end{equation}
\end{lemma}

\begin{proof}
For every $x\in\mathbb R$,
$|x|\leq\sqrt{2\delta h_\delta(x)}+h_\delta(x)$.
When $|x|\leq\delta$, the square-root term equals $|x|$.
When $|x|>\delta$, $h_\delta(x)\geq\delta/2$, so the square-root
term is at least $\delta$ and the inequality follows.
Summing over coordinates and applying Cauchy--Schwarz gives
\begin{equation}
\|\mathbf u-\mathbf v\|_1
\leq D\Psi_\delta\!\left(\rho_\delta(\mathbf u,\mathbf v)\right).
\label{eq:smooth_l1_vector_control}
\end{equation}
The two loss terms satisfy
$\rho_\delta(\mathbf c_t,\mathbf c_r)\leq\ell_{i,r\to t}$ and
$\rho_\delta(\widehat{\mathbf c}_t,\widehat{\mathbf c}_r)\leq\ell_{i,r\to t}/\beta$.
Monotonicity of $\Psi_\delta$ proves Eq.~\ref{eq:cmcdr_loss_to_mcd_bound}.
Its concavity and Jensen's inequality prove
Eq.~\ref{eq:cmcdr_average_loss_to_mcd_bound}.
\end{proof}

\paragraph{Prediction retention.}
For a replayed old sample, set $a=y_i$, $r=t-1$, and
$\mathcal A=\mathcal Y_{<t}$. Write
$\Gamma(u)=D\Psi_\delta(u)$ and let $\gamma_{i,t}$ be the logit margin
over the strongest competitor in $\mathcal A$.
Proposition~\ref{prop:contribution_margin_loss} and
Lemma~\ref{lem:cmcdr_controls_mcd} give
\begin{equation}
\begin{aligned}
\gamma_{i,t}
&\geq m^{\mathrm{keep}}_{i,t-1\to t}-\Gamma(\ell_{i,t-1\to t}),\\
m^{\mathrm{keep}}_{i,t-1\to t}
&=B_{i,t-1}(\mathcal M)
+B_{i,t}(\emptyset)-B_{i,t-1}(\emptyset).
\end{aligned}
\label{eq:cmcdr_retention_bound}
\end{equation}
Thus, $\Gamma(\ell_{i,t-1\to t})<m^{\mathrm{keep}}_{i,t-1\to t}$
ensures $\gamma_{i,t}>0$ within $\mathcal Y_{<t}$.
For prediction over $\mathcal Y_{\leq t}$,
Eq.~\ref{eq:full_space_retention_lower_bound} gives the sufficient condition
\begin{equation}
\Gamma(\ell_{i,t-1\to t})+q_{i,t}(\mathcal M)
<m^{\mathrm{keep}}_{i,t-1\to t},
\label{eq:cmcdr_full_space_retention}
\end{equation}
where $q_{i,t}$ is defined in
Eq.~\ref{eq:later_class_competition_decomposition} with
$\mathcal A_i=\mathcal Y_{<t}$ and
$\mathcal N_t=\mathcal Y_{\leq t}\setminus\mathcal Y_{<t}$.

\paragraph{Replay-free transfer.}
Fix the current and frozen models. Let $\mathcal D_t$ be the uniform
empirical distribution over $\mathcal X_{\mathrm{new}}$ and let
$U(\mathcal Y_{<t})$ be uniform over old classes. The probe distribution is
\begin{equation}
\nu_t=\mathcal D_t\otimes U(\mathcal Y_{<t}),
\qquad
\mathcal L_{\mathrm{cmcdr}}^{\mathrm{free}}
=\mathbb E_{(\mathbf x,a)\sim\nu_t}[\ell_{t-1\to t}(\mathbf x,a)],
\label{eq:replay_free_loss_expectation}
\end{equation}
where $\ell_{t-1\to t}(\mathbf x,a)$ is the per-profile loss under
$\mathcal A=\mathcal Y_{<t}$. Let $\mu_t$ be the old-task evaluation
distribution over input--label pairs and define
\begin{equation}
f_t(\mathbf x,a)
=\|\mathbf C_t^a(\mathbf x)-\mathbf C_{t-1}^a(\mathbf x)\|_1.
\label{eq:replay_free_transfer_drift}
\end{equation}
Equip $\mathcal X\times\mathcal Y_{<t}$ with a fixed metric $d_t$.
Let $W_1^{d_t}$ be the induced Wasserstein-1 distance~\citep{peyre2019computational}.

\begin{proposition}[Replay-free contribution drift]
\label{prop:replay_free_drift_bound}
Assume that $\mu_t$ and $\nu_t$ have finite first moments and that
$f_t$ is $L_t$-Lipschitz with respect to $d_t$. Then
\begin{equation}
\mathbb E_{\mu_t}[f_t]
\leq\Gamma\!\left(\mathcal L_{\mathrm{cmcdr}}^{\mathrm{free}}\right)
+L_tW_1^{d_t}(\mu_t,\nu_t).
\label{eq:replay_free_old_distribution_bound}
\end{equation}
\end{proposition}

\begin{proof}
Lemma~\ref{lem:cmcdr_controls_mcd} with $\pi=\nu_t$ gives
$\mathbb E_{\nu_t}[f_t]\leq
\Gamma(\mathcal L_{\mathrm{cmcdr}}^{\mathrm{free}})$.
For any coupling $\zeta$ of $\mu_t$ and $\nu_t$,
\[
\left|\mathbb E_{\mu_t}[f_t]-\mathbb E_{\nu_t}[f_t]\right|
\leq\mathbb E_{(Z,Z')\sim\zeta}|f_t(Z)-f_t(Z')|
\leq L_t\mathbb E_{(Z,Z')\sim\zeta}d_t(Z,Z').
\]
Taking the infimum over couplings proves the result.
\end{proof}

Appendix~\ref{app:probe_transfer} compares MCD scores on held-out current-task
and old-task probes to assess how closely their contribution drift agrees.

\begin{corollary}[Replay-free old-class forgetting]
\label{cor:replay_free_forgetting_bound}
Under the assumptions of Proposition~\ref{prop:replay_free_drift_bound}, let
$\gamma_{i,s}$ be the target-class logit margin over
$\mathcal A=\mathcal Y_{<t}$ at $s\in\{t-1,t\}$ for
$(\mathbf x_i,y_i)\sim\mu_t$. Define
$\mathcal F_t=\{\gamma_{i,t-1}>0,\ \gamma_{i,t}\leq0\}$.
For any $m_0>0$,
\begin{equation}
\Pr_{\mu_t}(\mathcal F_t)
\leq\Pr_{\mu_t}\!\left(m^{\mathrm{keep}}_{i,t-1\to t}\leq m_0\right)
+\frac{\Gamma(\mathcal L_{\mathrm{cmcdr}}^{\mathrm{free}})
+L_tW_1^{d_t}(\mu_t,\nu_t)}{m_0}.
\label{eq:replay_free_forgetting_bound}
\end{equation}
\end{corollary}

\begin{proof}
On $\mathcal F_t\cap\{m^{\mathrm{keep}}_{i,t-1\to t}>m_0\}$,
Proposition~\ref{prop:contribution_margin_loss} implies
$f_t(\mathbf x_i,y_i)>m_0$. The union bound and Markov's inequality give
$\Pr_{\mu_t}(\mathcal F_t)\leq
\Pr_{\mu_t}(m^{\mathrm{keep}}_{i,t-1\to t}\leq m_0)
+\mathbb E_{\mu_t}[f_t]/m_0$.
Apply Proposition~\ref{prop:replay_free_drift_bound}.
\end{proof}

\paragraph{Relation to cumulative MCD.}
For a fixed sample, target class, candidate set, and stage pair,
MCD and raw contribution drift differ by at most a factor of three.

\begin{lemma}[MCD and contribution drift]
\label{lem:mcd_controls_contribution_shift}
Fix $i$, $a$, $\mathcal A$, $r$, and $t$ as in
Eq.~\ref{eq:cmcdr_theory_profiles}, and let
$d^{\mathrm{MCD}}_{i,r\to t}
=|m_t-m_r|+(\min(m_t,m_r)+\epsilon)d^{\mathrm{rel}}_{i,r\to t}$.
Then
\begin{equation}
d^{\mathrm{abs}}_{i,r\to t}
\leq d^{\mathrm{MCD}}_{i,r\to t}
\leq3d^{\mathrm{abs}}_{i,r\to t}
\leq3\Gamma(\ell_{i,r\to t}).
\label{eq:cmcdr_mcd_equivalence}
\end{equation}
The centered responses $\bar B_{i,s}^{a,\mathcal A}(\mathcal S)
=B_{i,s}^{a,\mathcal A}(\mathcal S)-B_{i,s}^{a,\mathcal A}(\emptyset)$ satisfy
\begin{equation}
\begin{aligned}
\bar B_{i,t}^{a,\mathcal A}(\mathcal M)
-\bar B_{i,r}^{a,\mathcal A}(\mathcal M)
&=\mathbf 1^\top(\mathbf c_t-\mathbf c_r),\\
\max_{\mathcal S\subseteq\mathcal M}
\left|\bar B_{i,t}^{a,\mathcal A}(\mathcal S)
-\bar B_{i,r}^{a,\mathcal A}(\mathcal S)\right|
&\leq d^{\mathrm{abs}}_{i,r\to t}.
\end{aligned}
\label{eq:aggregate_shift_mcd_bound}
\end{equation}
The profiles are equal if and only if all centered subset responses are equal.
\end{lemma}

\begin{proof}
Assume $m_t\geq m_r$ without loss of generality and let $\Delta m=m_t-m_r$.
Since $\|\widehat{\mathbf c}_t\|_1\leq1$, the identity
\[
\mathbf c_t-\mathbf c_r
=(m_r+\epsilon)(\widehat{\mathbf c}_t-\widehat{\mathbf c}_r)
+\Delta m\widehat{\mathbf c}_t
\]
gives $d^{\mathrm{abs}}_{i,r\to t}\leq d^{\mathrm{MCD}}_{i,r\to t}$.
Rearranging the same identity gives
$(m_r+\epsilon)d^{\mathrm{rel}}_{i,r\to t}
\leq d^{\mathrm{abs}}_{i,r\to t}+\Delta m$.
The reverse triangle inequality yields
$\Delta m\leq d^{\mathrm{abs}}_{i,r\to t}$, hence
$d^{\mathrm{MCD}}_{i,r\to t}\leq3d^{\mathrm{abs}}_{i,r\to t}$.
Lemma~\ref{lem:cmcdr_controls_mcd} completes
Eq.~\ref{eq:cmcdr_mcd_equivalence}.
M\"obius inversion gives
\[
\bar B_{i,t}^{a,\mathcal A}(\mathcal S)
-\bar B_{i,r}^{a,\mathcal A}(\mathcal S)
=\sum_{\emptyset\neq\mathcal T\subseteq\mathcal S}
\left[C_{i,t}^{a,\mathcal A}(\mathcal T)
-C_{i,r}^{a,\mathcal A}(\mathcal T)\right].
\]
Taking $\mathcal S=\mathcal M$ and bounding each subset sum proves
Eq.~\ref{eq:aggregate_shift_mcd_bound}. Invertibility of the transform
gives the equivalence of profile and centered-response equality.
\end{proof}

Equation~\ref{eq:mcd} measures drift from task introduction under
$\mathcal A_y$, whereas each training update uses $\mathcal Y_{<t}$.
Cumulative drift is bounded by the sum of per-update drift and
candidate-set residuals.

\begin{corollary}[MCD after multiple task updates]
\label{cor:fixed_reference_update_bound}
Fix an old-task probe $i$ with label $y$ introduced at stage $k$ and let
$A=\mathcal A_y$, with $|A|\geq2$. Use the same modality-removal rule
at stages $k,k+1,\ldots,t$. For each update $u\in\{k+1,\ldots,t\}$,
let $B_u=\mathcal Y_{<u}\supseteq A$ and define, for $v\in\{u-1,u\}$,
\begin{equation}
\begin{aligned}
q_{i,v}^{A,B_u}(\mathcal S)
&=\log\!\left(1+
\frac{\sum_{c\in B_u\setminus A}\exp z_{i,v,c}(\mathcal S)}
{\sum_{c\in A\setminus\{y\}}\exp z_{i,v,c}(\mathcal S)}\right),\\
\widetilde q_{i,v}^{A,B_u}(\mathcal S)
&=q_{i,v}^{A,B_u}(\mathcal S)-q_{i,v}^{A,B_u}(\emptyset).
\end{aligned}
\label{eq:old_class_competition_residual}
\end{equation}
Let $\mathsf{Mob}$ be the M\"obius transform over nonempty coalitions and set
$R_{i,u}^{A}=\|\mathsf{Mob}(\widetilde q_{i,u}^{A,B_u}
-\widetilde q_{i,u-1}^{A,B_u})\|_1$.
Let $\ell_{i,u}^{B_u}$ be the loss in
Eq.~\ref{eq:cmcdr_per_profile_loss} evaluated on this probe with target
$y$, candidate set $B_u$, and stages $(u-1,u)$. Then
\begin{equation}
d^{\mathrm{MCD},A}_{i,k\to t}
\leq3\sum_{u=k+1}^{t}
\left[\Gamma(\ell_{i,u}^{B_u})+R_{i,u}^{A}\right],
\label{eq:cmcdr_cumulative_mcd_bound}
\end{equation}
where $d^{\mathrm{MCD},A}_{i,k\to t}$ is the per-probe score in Eq.~\ref{eq:mcd}.
\end{corollary}

\begin{proof}
Write $\mathbf c_{i,v}^{A}$ and $\mathbf c_{i,v}^{B_u}$ for the profiles
with target $y$ under the respective candidate sets.
Equation~\ref{eq:later_class_competition_decomposition} and linearity of
the M\"obius transform give
\begin{equation}
\mathbf c_{i,u}^{A}-\mathbf c_{i,u-1}^{A}
=\mathbf c_{i,u}^{B_u}-\mathbf c_{i,u-1}^{B_u}
+\mathsf{Mob}(\widetilde q_{i,u}^{A,B_u}
-\widetilde q_{i,u-1}^{A,B_u}).
\label{eq:fixed_reference_update_bound}
\end{equation}
Thus, $\|\mathbf c_{i,u}^{A}-\mathbf c_{i,u-1}^{A}\|_1
\leq\Gamma(\ell_{i,u}^{B_u})+R_{i,u}^{A}$ by
Lemma~\ref{lem:cmcdr_controls_mcd}.
Summing over updates bounds $\|\mathbf c_{i,t}^{A}-\mathbf c_{i,k}^{A}\|_1$.
\end{proof}

The residual $R_{i,u}^{A}$ measures changes in competition from classes in
$B_u\setminus A$. It vanishes when
$q_{i,u}^{A,B_u}(\mathcal S)-q_{i,u-1}^{A,B_u}(\mathcal S)$ is constant
across modality subsets.
Averaging Eq.~\ref{eq:cmcdr_cumulative_mcd_bound} over the fixed old-task
probes with the task weights in Eq.~\ref{eq:mcd} gives a bound on
$\mathrm{MCD}_t$. Each $\ell_{i,u}^{B_u}$ is evaluated on the corresponding
old-task probe.

\subsection{M\"obius Matching vs. Centered-Response Matching}
\label{app:interaction_regularization}

The absolute term of CMCDR penalizes changes in each modality contribution
and interaction separately. Centered-response matching penalizes their sum
within each modality subset, where opposing changes can cancel. Although
the two complete profiles contain the same information, their matching
losses assign different costs to these changes.

Fix the sample, target class, candidate set, and removal rule
as in Appendix~\ref{app:cmcdr_theory}. For stages $s\in\{r,t\}$, write
$b_s(\mathcal S)=B_s(\mathcal S)-B_s(\emptyset)$,
$c_s(\mathcal T)=C_s(\mathcal T)$, and
$d_{\mathcal T}=c_t(\mathcal T)-c_r(\mathcal T)$.

\begin{proposition}[Penalty for offsetting interaction changes]
\label{prop:interaction_regularization}
Let $D=2^{|\mathcal M|}-1$ and let $h_\delta$ be the Smooth-$L_1$ function
in Eq.~\ref{eq:cmcdr_smooth_l1_definition}. The absolute matching losses are
\begin{equation}
\mathcal L_{\mathrm{center}}^{\mathrm{abs}}
=\frac{1}{D}\sum_{\emptyset\neq\mathcal S\subseteq\mathcal M}
h_\delta\!\left(\sum_{\emptyset\neq\mathcal T\subseteq\mathcal S}d_{\mathcal T}\right),
\qquad
\mathcal L_{\mathrm{Mob}}^{\mathrm{abs}}
=\frac{1}{D}\sum_{\emptyset\neq\mathcal T\subseteq\mathcal M}h_\delta(d_{\mathcal T}).
\label{eq:mobius_matching_geometry}
\end{equation}
Both losses vanish exactly when $\mathbf c_t=\mathbf c_r$.
For two modalities, let $\Delta b_{12}=d_1+d_2+d_{12}$ be the change
in the joint centered response. Then
\begin{equation}
\mathcal L_{\mathrm{Mob}}^{\mathrm{abs}}
-\mathcal L_{\mathrm{center}}^{\mathrm{abs}}
=\frac{h_\delta(d_{12})-h_\delta(\Delta b_{12})}{3}.
\label{eq:interaction_cancellation_penalty}
\end{equation}
Thus, M\"obius matching assigns a strictly larger absolute loss whenever
$|d_{12}|>|\Delta b_{12}|$.
\end{proposition}

\begin{proof}
M\"obius inversion gives
$b_t(\mathcal S)-b_r(\mathcal S)=\sum_{\emptyset\neq\mathcal T\subseteq\mathcal S}d_{\mathcal T}$.
Substitution gives Eq.~\ref{eq:mobius_matching_geometry}. Invertibility
of the M\"obius transform and $h_\delta(x)=0$ if and only if $x=0$
give the common zero-loss condition. For two modalities,
$\Delta b_1=d_1$, $\Delta b_2=d_2$, and
$\Delta b_{12}=d_1+d_2+d_{12}$. The singleton terms in the two losses
coincide, leaving Eq.~\ref{eq:interaction_cancellation_penalty}.
The strict inequality follows because $h_\delta(x)$ is strictly
increasing in $|x|$.
\end{proof}

For example, $\Delta\mathbf c=(\eta,\eta,-2\eta)$ with $\eta>0$
gives $\Delta\mathbf b=(\eta,\eta,0)$. Centered-response matching
penalizes the two singleton changes, while M\"obius matching adds
$h_\delta(2\eta)/3$ for the interaction decrease. For
$0<\eta\leq\delta/2$, the absolute losses are $\eta^2/(3\delta)$ and
$\eta^2/\delta$, respectively. M\"obius matching therefore imposes
three times the penalty on this update, whose joint centered response
remains unchanged.

M\"obius matching discourages changes in individual contributions even
when they cancel in the joint centered response.
Table~\ref{tab:mobius_parameterization_ablation} evaluates prediction
retention under both representations, using the absolute and relative
terms in Eq.~\ref{eq:cmcdr_profile_distance}, identical modality subsets,
and the same tuning budget. Across both datasets and learning protocols,
M\"obius matching improves average accuracy by 1.15$\mathord{\sim}$1.73 percentage points
and reduces average forgetting by 1.53$\mathord{\sim}$2.06 percentage points.

\section{Experimental Settings and Implementation Details}
\label{app:experimental_details}

\subsection{Experimental Settings for Exploratory Analyses}
\label{app:experimental_settings}

We detail the settings for the exploratory analyses in
Section.~\ref{sec:contribution_drift}.

\phantomsection
\label{app:exploratory_setup}

\noindent\textbf{Modality-contribution drift analysis.}
We analyze whether iCaRL preserves old-task modality contributions on AVE and
UESTC-MMEA-CL. AVE contains audio and visual modalities, while
UESTC-MMEA-CL contains video, acceleration, and gyroscope streams.

For AVE, visual and audio features are extracted using pretrained VideoMAE~\citep{tong2022videomae} and
AudioMAE~\citep{huang2022audiomae} models. For UESTC-MMEA-CL, the video stream follows the same visual
pipeline, while acceleration and gyroscope signals are normalized and encoded
by a CNN-LSTM. Modality features are projected into a shared space before
fusion and classification.

Both datasets follow fixed class orders throughout the experiment. The model is
trained sequentially from the previous-stage checkpoint using the same iCaRL
configuration at each stage. After each update, we report the mean accuracy over
observed tasks and compute the MCD score over old classes using the contribution estimator
in Section.~\ref{sec:causal_contribution}. These stage-wise results are reported
separately from the aggregate results in
Tables~\ref{tab:class_incremental_results} and
\ref{tab:rep_preservation_complementarity}.

The exploratory analyses use four tasks on AVE and eight tasks on
UESTC-MMEA-CL. At each stage, all test samples from previously learned classes
serve as diagnostic probes. For each probe from task $k$, we compare its
current contribution profile with the profile recorded immediately after
learning task $k$, using the fixed candidate set
$\mathcal A_{y_i}=\mathcal Y_{\leq\tau(y_i)}$ for both profiles.
We aggregate these cumulative changes according to Eq.~\ref{eq:mcd}.

\noindent\textbf{Cross-modal stability analysis.}
We compare iCaRL with RepAlign~\citep{pian2023avcil} and
CrossSDC~\citep{pian2024contavsep} under the same backbone, task order,
memory budget, optimizer, and training schedule. RepAlign preserves
instance- and class-level audio--visual similarity, while CrossSDC further
preserves cross-modal similarities between the current and previous models.
Both constraints are added to the original iCaRL objective.

We measure representation drift using linear CKA~\citep{kornblith2019similarity}
on the same old-class probes used for the MCD score. For task $k$ and modality
$m$, let $\mathbf H_{k,s}^m\in\mathbb R^{|\mathcal P_k^{\mathrm{task}}|\times d}$
contain the column-centered projected representations of the fixed probes
in $\mathcal P_k^{\mathrm{task}}$ after learning stage $s$, in the same row order.
We compare the representations recorded immediately after learning task $k$
with those at stage $t$:
\begin{equation}
\begin{aligned}
\operatorname{CKA}(\mathbf H_{k,k}^{m},\mathbf H_{k,t}^{m})
&=
\frac{
\|{\mathbf H_{k,k}^{m}}^{\top}\mathbf H_{k,t}^{m}\|_F^2
}{
\|{\mathbf H_{k,k}^{m}}^{\top}\mathbf H_{k,k}^{m}\|_F
\|{\mathbf H_{k,t}^{m}}^{\top}\mathbf H_{k,t}^{m}\|_F
},\\
D_t^{\mathrm{rep}}
&=
1-\frac{1}{M|\mathcal K_{<t}|}
\sum_{k\in\mathcal K_{<t}}\sum_{m=1}^{M}
\operatorname{CKA}(\mathbf H_{k,k}^{m},\mathbf H_{k,t}^{m}).
\end{aligned}
\label{eq:cross_modal_cka_drift}
\end{equation}
Lower $D_t^{\mathrm{rep}}$ indicates smaller cumulative representation drift
from task introduction to stage $t$. The MCD score measures cumulative
contribution drift over the same interval on the same probes under each task's fixed
candidate set, as defined in Eq.~\ref{eq:mcd}. Both metrics are evaluated
after each incremental update and averaged within each dataset.

For Figure~\ref{stability_panels}(d), the residual MCD score of each class is averaged
across incremental stages and across RepAlign and CrossSDC. Classes are divided into
low- and high-MCD groups using the within-dataset median. The same groups are
used for all methods.

\subsection{Main Training and Evaluation Setup}
\label{app:main_training_setup}

\paragraph{Task construction.}
All methods follow the published train/test splits, task construction,
and task order for each benchmark.
For multimodal class-incremental learning, each benchmark is divided into a
fixed sequence of disjoint class groups, introduced one per stage.
AVE uses four tasks of seven classes, and Kinetics-Sounds uses five tasks
of six classes, following AV-CIL~\citep{pian2023avcil}.
UESTC-MMEA-CL uses eight tasks of four classes~\citep{xu2024uestc}.
For continual multimodal question answering, VQAv2 follows the VQACL
protocol~\citep{zhang2023vqacl}, while Split-AVQA and Split-MUSIC-AVQA
follow the AVQACL protocol~\citep{wu2025avqacl}.

\paragraph{Model and baseline implementation.}
We use each method's released codebase when available, retaining its
architecture and baseline-specific components. For class-incremental tasks,
we implement classical continual-learning methods and methods without
publicly released code in the AV-CIL~\citep{pian2023avcil} codebase.
This framework extracts visual and audio features using pretrained
VideoMAE~\citep{tong2022videomae} and AudioMAE~\citep{huang2022audiomae},
respectively. On UESTC-MMEA-CL, acceleration and gyroscope signals are
normalized and encoded by CNN-LSTM.
Within each baseline,
the original method and its CMCDR variant use the same architecture,
memory budget, task order, and inference procedure.

\paragraph{Training.}
All methods are trained sequentially, starting with the initial task at
$t=1$. At subsequent stages, replay-based methods use current-task data
and their exemplar buffer, while replay-free methods use only current-task
data. CMCDR compares the current model with the frozen previous model,
using buffered old exemplars in the replay-based setting and current-task
samples in the replay-free setting. All contribution profiles use the
zero-representation removal in Appendix~\ref{app:modality_removal}.

Gradient-trained baselines initialize each subsequent stage from the previous
checkpoint. We retain the original training loss, optimizer, learning-rate
schedule, batch size, number of epochs, and exemplar-selection rule,
with the same settings across incremental stages. CMCDR adds the
contribution-preservation regularizer in Section.~\ref{sec:method} to the
original objective.

MMAL updates the fusion and modality-specific compensation branches
analytically at each incremental stage. We then initialize separate
prediction-head copies from the resulting solutions. With the feature
extractors frozen, we optimize these copies on current-task data using the
normalized ridge objectives and replay-free CMCDR. Contribution profiles are
computed from the combined branch scores. The refined heads are used for
inference and as the frozen reference at the next stage.

\paragraph{Hyperparameters.}
We use $\lambda_{\mathrm{cmcdr}}=\beta=1.0$ in the main experiments,
with $\epsilon=10^{-6}$ for contribution-profile normalization and
$\delta=1.0$ as the Smooth-$L_1$ transition threshold.
Appendix~\ref{app:hyperparameter_analysis} reports sensitivity analyses for
$\lambda_{\mathrm{cmcdr}}\in\{0.1,0.5,1.0,2.0\}$ and
$\beta\in\{0.25,0.5,1.0,2.0\}$.

\paragraph{Evaluation.}
Class-incremental recognition is evaluated over all classes observed so far,
without task identity at inference. MCD measures cumulative contribution drift
using old-task test probes and their contribution profiles recorded at task
introduction. These probes and recorded profiles are used only for evaluation.
Both CMCDR variants and the baselines use the fixed-reference MCD protocol
in Eq.~\ref{eq:mcd}.

\subsection{CMCDR for Continual Multimodal Question Answering}
\label{app:vqa_cmcdr}

We implement LwF~\citep{li2017lwf},
EWC~\citep{kirkpatrick2017ewc}, and
iCaRL~\citep{rebuffi2017icarl} in the official AVQACL
codebase~\citep{wu2025avqacl}. On each benchmark, these
three baselines share the modality encoders, fusion module,
and classification-based answer predictor. They differ in
their continual-learning objectives and memory mechanisms.
Each baseline and its CMCDR variant use the same architecture,
answer vocabulary, optimizer, training schedule, memory budget,
and evaluation protocol.

For Split-AVQA and Split-MUSIC-AVQA, we follow the AVQACL
protocol and retain audio, visual, and question inputs.
Their answer vocabularies contain 2,135 and 40 classes,
respectively. CMCDR applies the modality-removal rule in
Appendix~\ref{app:modality_removal}.

Let $\mathbf{x}_i=(x_i^m)_{m\in\mathcal M}$ denote an input
and $a_i$ its reference answer. At stage $s\in\{t-1,t\}$,
$z_{i,s,a}(\mathcal S)$ denotes the score of candidate answer $a$
when modalities $\mathcal S\subseteq\mathcal M$ are retained.
The candidate set $\mathcal A_{<t}$ is shared by the current and frozen
models and remains fixed across all modality subsets within each stage.

In classification-based QA, $z_{i,s,a}(\mathcal S)$ is the logit
of answer class $a$, and $\mathcal A_{<t}$ contains
the answer classes represented by the previous-stage
prediction head, with $|\mathcal A_{<t}|\geq 2$.
For a fixed answer vocabulary, $\mathcal A_{<t}$ is the
full vocabulary; for an expanding head, it contains the
classes included before stage $t$.

For generative VQACL, we retain the original decoder and score each answer
sequence $a=(a_1,\ldots,a_{L_a})$ by teacher forcing:
\begin{equation}
z_{i,s,a}(\mathcal S)
=\frac{1}{L_a}\sum_{j=1}^{L_a}
\log p_{\theta_s}\!\left(a_j\mid a_{<j},\mathbf{x}_i^{\mathcal S}\right),
\label{eq:generative_answer_score}
\end{equation}
where $L_a$ is the number of scored answer tokens. The candidate set
$\mathcal A_{<t}$ contains the distinct normalized reference answers
in the old-task replay buffer.

For both prediction formulations, we compute the answer benefit using
Eq.~\ref{eq:benefit_function} with target answer $a\in\mathcal A_{<t}$
and candidate set $\mathcal A=\mathcal A_{<t}$ at each stage $s$.
The M\"obius transform in Eq.~\ref{eq:general_contribution} then yields
$C_{i,s}^{a}(\mathcal T)$ for each nonempty coalition
$\mathcal T\subseteq\mathcal M$.
We collect these terms into
$\mathbf C_{i,s}^{a}
=[C_{i,s}^{a}(\mathcal T)]_{\emptyset\neq\mathcal T
\subseteq\mathcal M}$, with a fixed coalition ordering.

\paragraph{Replay-based regularization.}
Let $\mathcal X_{\mathrm{old}}$ contain the replayed
input--answer pairs. CMCDR matches the contribution
profiles of their ground-truth answers using
Eq.~\ref{eq:replay_cmcdr_loss}, with $y_i$ replaced by
$a_i\in\mathcal A_{<t}$ and profiles computed over
$\mathcal A_{<t}$. The distance $d_{\mathrm{cp}}$ is defined in
Eq.~\ref{eq:cmcdr_profile_distance}.

\paragraph{Replay-free regularization.}
For replay-free answer classifiers, current-stage inputs
$\mathcal X_{\mathrm{new}}$ serve
as intervention probes. For each input, CMCDR matches
the contribution profiles of all answer classes in
$\mathcal A_{<t}$ using Eq.~\ref{eq:free_cmcdr_loss}, with
$\mathcal Y_{<t}$ replaced by $\mathcal A_{<t}$ and the class index $k$
by the answer index $a$. Here $N_t=|\mathcal X_{\mathrm{new}}|$.

Both variants evaluate all $2^{|\mathcal M|}$ modality
subsets using the same removal rule for the current
and frozen models. Gradients are stopped through the
frozen model. We use the training objective in
Eq.~\ref{eq:overall_objective}, with $\mathcal L_{\mathrm{cmcdr}}$ set to
the replay-based or replay-free regularizer above.
The baseline answer prediction and evaluation
procedures remain unchanged.

\subsection{Modality-Removal Realizations}
\label{app:modality_removal}

\begin{table}[htbp]
\centering
\small
\setlength{\tabcolsep}{4pt}
\caption{Robustness to different modality-removal realizations.}
\label{tab:removal_sensitivity_protocol}
\begin{tabular}{lllccc}
\toprule
\multicolumn{6}{c}{\textbf{AVE}} \\
\cmidrule(lr){1-6}
\textbf{Protocol} & \textbf{Method} & \textbf{Removal} &
\textbf{Avg. Acc. $\uparrow$} & \textbf{Avg. Forget. $\downarrow$} &
\textbf{MCD $\downarrow$} \\
\midrule
Replay & iCaRL & Zero & 64.20 & 25.60 & 0.779 \\
\rowcolor{cmcdrgreen}
Replay & \textbf{iCaRL + CMCDR} & Zero & 71.86 & 17.36 & 0.389 \\
Replay & iCaRL & Reference mean & 64.20 & 25.60 & 0.812 \\
\rowcolor{cmcdrgreen}
Replay & \textbf{iCaRL + CMCDR} & Reference mean & 71.63 & 17.82 & 0.417 \\
Replay & iCaRL & Reference anchor & 64.20 & 25.60 & 0.864 \\
\rowcolor{cmcdrgreen}
Replay & \textbf{iCaRL + CMCDR} & Reference anchor & 71.34 & 18.07 & 0.436 \\
\midrule
Replay-free & LwF & Zero & 56.61 & 35.11 & 0.842 \\
\rowcolor{cmcdrgreen}
Replay-free & \textbf{LwF + CMCDR} & Zero & 60.61 & 26.35 & 0.401 \\
Replay-free & LwF & Reference mean & 56.61 & 35.11 & 0.879 \\
\rowcolor{cmcdrgreen}
Replay-free & \textbf{LwF + CMCDR} & Reference mean & 60.28 & 26.84 & 0.423 \\
Replay-free & LwF & Reference anchor & 56.61 & 35.11 & 0.817 \\
\rowcolor{cmcdrgreen}
Replay-free & \textbf{LwF + CMCDR} & Reference anchor & 60.03 & 27.16 & 0.451 \\
\midrule
\multicolumn{6}{c}{\textbf{UESTC-MMEA-CL}} \\
\cmidrule(lr){1-6}
\textbf{Protocol} & \textbf{Method} & \textbf{Removal} &
\textbf{Avg. Acc. $\uparrow$} & \textbf{Avg. Forget. $\downarrow$} &
\textbf{MCD $\downarrow$} \\
\midrule
Replay & iCaRL & Zero & 71.83 & 22.86 & 0.464 \\
\rowcolor{cmcdrgreen}
Replay & \textbf{iCaRL + CMCDR} & Zero & 75.12 & 17.15 & 0.288 \\
Replay & iCaRL & Reference mean & 71.83 & 22.86 & 0.493 \\
\rowcolor{cmcdrgreen}
Replay & \textbf{iCaRL + CMCDR} & Reference mean & 74.96 & 17.58 & 0.307 \\
Replay & iCaRL & Reference anchor & 71.83 & 22.86 & 0.538 \\
\rowcolor{cmcdrgreen}
Replay & \textbf{iCaRL + CMCDR} & Reference anchor & 74.73 & 17.94 & 0.332 \\
\midrule
Replay-free & LwF & Zero & 17.13 & 49.37 & 0.513 \\
\rowcolor{cmcdrgreen}
Replay-free & \textbf{LwF + CMCDR} & Zero & 22.85 & 38.65 & 0.299 \\
Replay-free & LwF & Reference mean & 17.13 & 49.37 & 0.557 \\
\rowcolor{cmcdrgreen}
Replay-free & \textbf{LwF + CMCDR} & Reference mean & 22.47 & 39.14 & 0.319 \\
Replay-free & LwF & Reference anchor & 17.13 & 49.37 & 0.526 \\
\rowcolor{cmcdrgreen}
Replay-free & \textbf{LwF + CMCDR} & Reference anchor & 22.16 & 39.63 & 0.347 \\
\bottomrule
\end{tabular}
\end{table}

Equation~\ref{eq:modality_intervention} removes excluded modalities through an
absence symbol. For fixed-interface models, we realize this intervention at the
fusion input. Let
$\mathbf h_{i,s}^{m}=E_{m,s}(x_i^m)$ denote the representation of modality $m$
at stage $s$. Given a replacement vector $\mathbf q_s^m$, we define
\begin{equation}
\widetilde{\mathbf h}_{i,s}^{m,\mathbf q}(\mathcal S)
=
\begin{cases}
\mathbf h_{i,s}^{m}, & m\in\mathcal S,\\
\mathbf q_s^m, & m\notin\mathcal S,
\end{cases}
\qquad
\mathbf z_{i,s}^{\mathbf q}(\mathcal S)
=
G_{\theta_s}
\left(
\{\widetilde{\mathbf h}_{i,s}^{m,\mathbf q}(\mathcal S)\}_{m=1}^{M}
\right).
\label{eq:removal_realization}
\end{equation}
The same removal rule is applied to the current and frozen models.

\paragraph{Zero-representation removal.}
The default realization in the main experiments uses $\mathbf q_s^m=\mathbf 0$
for each excluded modality. This replacement is deterministic, sample-independent,
reference-free, and identical across incremental stages.

\paragraph{Reference-mean removal.}
We also consider the mean representation of a fixed unlabeled reference set
$\mathcal R^m$:
\begin{equation}
\mathbf q_s^m
=
\boldsymbol\mu_s^m
=
\frac{1}{|\mathcal R^m|}
\sum_{x_j^m\in\mathcal R^m}
E_{m,s}(x_j^m).
\label{eq:reference_mean_removal}
\end{equation}
Because the encoder changes across stages, the resulting reference
representation may also change.

\paragraph{Reference-anchor removal.}
Alternatively, for each modality, we select a fixed reference sample
$x_{r_m}^m\in\mathcal R^m$ and set $\mathbf q_s^m=E_{m,s}(x_{r_m}^m)$.
The reference sample is fixed across stages, while its encoded representation
changes with the encoder.

Reference-mean and reference-anchor removal require a fixed auxiliary unlabeled
reference set and are used only in the robustness experiments below.

\paragraph{Modality-removal robustness.}
We compare zero, reference-mean, and reference-anchor removal using iCaRL and
LwF, with and without CMCDR. For each CMCDR variant, the same removal rule is
used for contribution-profile construction, regularization, and evaluation,
 while all other training settings remain fixed. For the corresponding
 baselines, modality removal is applied only when computing the MCD score. Their
 accuracy and forgetting values are therefore identical across removal rules.

Table~\ref{tab:removal_sensitivity_protocol} reports average accuracy, average
 forgetting, and the MCD score. Across both datasets and learning protocols, CMCDR improves accuracy and
reduces forgetting and the MCD score under all three removal rules.

\section{Additional Experiments}
\label{app:additional_experiments}

\subsection{Random-Seed Robustness}
\label{app:seed_robustness}

\begin{table}[htbp]
\centering
\scriptsize
\setlength{\tabcolsep}{3.2pt}
\caption{Five-seed robustness results on AVE and UESTC-MMEA-CL (mean $\pm$ sample standard deviation).}
\label{tab:seed_robustness}
\resizebox{\textwidth}{!}{%
\begin{tabular}{llccc|ccc}
\toprule
\textbf{Protocol} & \textbf{Method} &
\multicolumn{3}{c|}{\textbf{AVE}} &
\multicolumn{3}{c}{\textbf{UESTC-MMEA-CL}} \\
\cmidrule(lr){3-5}\cmidrule(lr){6-8}
& & \textbf{Avg. Acc. $\uparrow$} & \textbf{Avg. Forget. $\downarrow$} & \textbf{MCD $\downarrow$} &
\textbf{Avg. Acc. $\uparrow$} & \textbf{Avg. Forget. $\downarrow$} & \textbf{MCD $\downarrow$} \\
\midrule
Replay-free & LwF & 56.61 \stddev{0.31} & 35.11 \stddev{0.28} & 0.842 \stddev{0.042} & 17.13 \stddev{0.44} & 49.37 \stddev{0.37} & 0.513 \stddev{0.029} \\
\rowcolor{cmcdrgreen}
Replay-free & \textbf{LwF + CMCDR} & \textbf{60.61} \stddev{0.56} & \textbf{26.35} \stddev{0.52} & \textbf{0.401} \stddev{0.032} & \textbf{22.85} \stddev{0.70} & \textbf{38.65} \stddev{0.60} & \textbf{0.299} \stddev{0.024} \\
Replay-free & EWC & 41.28 \stddev{0.27} & 68.97 \stddev{0.42} & 0.962 \stddev{0.057} & 19.57 \stddev{0.48} & 29.84 \stddev{0.22} & 0.482 \stddev{0.030} \\
\rowcolor{cmcdrgreen}
Replay-free & \textbf{EWC + CMCDR} & \textbf{47.36} \stddev{0.79} & \textbf{57.42} \stddev{0.88} & \textbf{0.375} \stddev{0.037} & \textbf{24.78} \stddev{0.68} & \textbf{22.06} \stddev{0.35} & \textbf{0.308} \stddev{0.026} \\
Replay-free & MMAL & 70.28 \stddev{0.35} & 19.64 \stddev{0.24} & 0.724 \stddev{0.038} & 75.57 \stddev{0.47} & 24.83 \stddev{0.51} & 0.458 \stddev{0.017} \\
\rowcolor{cmcdrgreen}
Replay-free & \textbf{MMAL + CMCDR} & \textbf{72.63} \stddev{0.50} & \textbf{15.02} \stddev{0.36} & \textbf{0.321} \stddev{0.030} & \textbf{78.12} \stddev{0.59} & \textbf{20.94} \stddev{0.63} & \textbf{0.319} \stddev{0.018} \\
\midrule
Replay & iCaRL & 64.20 \stddev{0.41} & 25.60 \stddev{0.37} & 0.779 \stddev{0.044} & 71.83 \stddev{0.52} & 22.86 \stddev{0.45} & 0.464 \stddev{0.028} \\
\rowcolor{cmcdrgreen}
Replay & \textbf{iCaRL + CMCDR} & \textbf{71.86} \stddev{0.82} & \textbf{17.36} \stddev{0.58} & \textbf{0.389} \stddev{0.026} & \textbf{75.12} \stddev{0.67} & \textbf{17.15} \stddev{0.58} & \textbf{0.288} \stddev{0.025} \\
Replay & AV-CIL & 71.37 \stddev{0.40} & 6.77 \stddev{0.28} & 0.649 \stddev{0.030} & 80.27 \stddev{0.58} & 18.74 \stddev{0.34} & 0.449 \stddev{0.029} \\
\rowcolor{cmcdrgreen}
Replay & \textbf{AV-CIL + CMCDR} & \textbf{75.86} \stddev{0.60} & \textbf{5.31} \stddev{0.37} & \textbf{0.227} \stddev{0.025} & \textbf{83.68} \stddev{0.72} & \textbf{13.36} \stddev{0.45} & \textbf{0.275} \stddev{0.026} \\
Replay & CIGN & 73.92 \stddev{0.45} & 10.84 \stddev{0.18} & 0.691 \stddev{0.039} & 82.14 \stddev{0.53} & 16.27 \stddev{0.31} & 0.461 \stddev{0.026} \\
\rowcolor{cmcdrgreen}
Replay & \textbf{CIGN + CMCDR} & \textbf{74.85} \stddev{0.42} & \textbf{7.11} \stddev{0.27} & \textbf{0.240} \stddev{0.026} & \textbf{83.75} \stddev{0.62} & \textbf{12.97} \stddev{0.38} & \textbf{0.254} \stddev{0.023} \\
\midrule
\end{tabular}}
\end{table}

We assess whether the gains of CMCDR persist under stochastic training on AVE
and UESTC-MMEA-CL for the class-incremental methods reported in
Table~\ref{tab:class_incremental_results}. For each setting, we use the same
five random seeds for a baseline and its CMCDR counterpart. Model
initialization, data order, and exemplar sampling are paired within a seed,
whereas the benchmark split, class order, memory budget, and the coefficients
selected for the main evaluation remain fixed. Table~\ref{tab:seed_robustness} reports the mean and sample standard deviation
of average accuracy, average forgetting, and the MCD score over the five runs.

\subsection{Modality Contribution Drift and Class-wise Forgetting}
\label{app:mcd_forgetting}

\paragraph{Class-wise effects of CMCDR.}
We compare class-wise MCD and forgetting on AVE for iCaRL and LwF,
with and without CMCDR. Class-wise MCD averages the per-probe score in
Eq.~\ref{eq:mcd} over each old class.
Figure~\ref{fig:ave_mcd_forgetting} shows 21 class pairs for each baseline.

CMCDR reduces MCD for all displayed classes and reduces forgetting for
19 of 21 iCaRL classes and 20 of 21 LwF classes.
The remaining classes have lower MCD but higher forgetting.
For one iCaRL class, MCD decreases from $0.391$ to $0.180$, while forgetting
increases from $17.01\%$ to $21.97\%$.
For the LwF exception, MCD decreases from $0.494$ to $0.241$, while forgetting
increases from $24.41\%$ to $28.24\%$.
CMCDR improves retention for most classes, with a few class-specific exceptions.

These exceptions indicate that reducing MCD does not necessarily reduce
forgetting for every class. One possible explanation is task similarity:
when new and old tasks share modality contribution patterns, changes in
modality-specific contributions and cross-modal interactions may benefit
old-class discrimination. Regularizing these changes may also restrict
beneficial adaptation. This explanation remains a hypothesis to be examined
in future work.

\begin{figure}[!htbp]
    \centering
    \includegraphics[width=\textwidth]{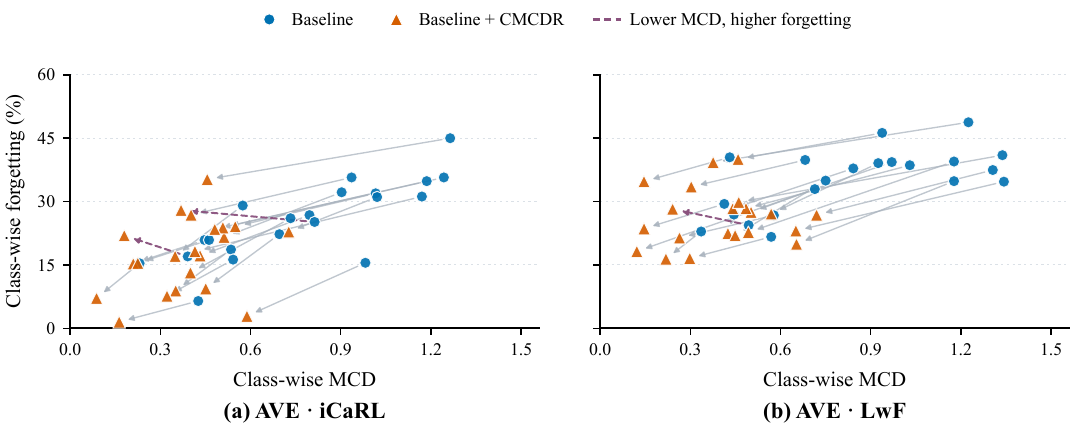}
    \caption{Class-wise MCD and forgetting on AVE.
    Arrows connect each class under the baseline and its CMCDR variant.
    Dashed purple arrows mark classes with lower MCD but higher forgetting.}
    \label{fig:ave_mcd_forgetting}
\end{figure}

\paragraph{MCD and forgetting under comparable full-input benefit drift.}
The class-wise comparisons above suggest an association between MCD and
forgetting, with some exceptions. We next examine whether this association
persists when changes in the overall full-input benefit are comparable.
We conduct this analysis separately for iCaRL and LwF on AVE, treating each
old class at each incremental stage as an observation.

For each observation, we use the same probe samples and fixed candidate set
as in the MCD evaluation. We first subtract the no-modality benefit from
the full-input benefit. For each probe, we take the absolute change in this
centered benefit from the class's introduction stage to the current stage,
and then average these changes over the class. This gives the cumulative
full-input benefit drift for that observation. For each method, we divide
the observed drift range into three equal-width bins and split observations
within each bin into low- and high-MCD groups using that bin's median MCD.
We compute the mean class-wise forgetting for each group within each bin,
and then average the three bin-level means with equal weights.

Figure~\ref{fig:mcd_benefit_drift} reports these equal-bin averages.
For iCaRL, mean forgetting increases from $63.8\%$ in the low-MCD group to
$74.5\%$ in the high-MCD group, a gap of $10.7$ percentage points. For LwF, the corresponding
values are $53.1\%$ and $65.1\%$, with a gap of $12.0$ percentage points.
These positive gaps support an association between higher MCD and greater
forgetting after stratification by centered full-input benefit drift.

\begin{figure}[!htbp]
    \centering
    \includegraphics[width=\textwidth]{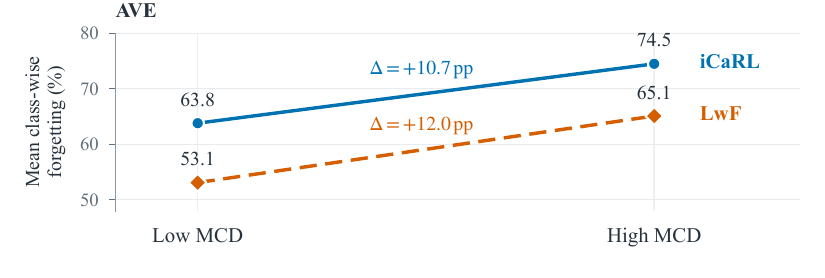}
    \caption{MCD and forgetting after stratification by full-input benefit
    drift on AVE. Class--stage observations are divided into three equal-width
    bins of centered full-input benefit drift and split at the median MCD
    within each bin. Each point is the equally weighted average of the three
    bin-level forgetting means for the indicated MCD group.}
    \label{fig:mcd_benefit_drift}
\end{figure}

\subsection{Computational Efficiency and Scalability}
\label{app:efficiency_scalability}

\begin{table}[htbp]
\centering
\scriptsize
\setlength{\tabcolsep}{4pt}
\caption{Mean training time per incremental stage on UESTC-MMEA-CL.
Overhead is relative to each baseline.
$\Delta_M$ is the CMCDR training-time increase from two to three modalities.
Acc. and Gyro. denote acceleration and gyroscope signals.}
\label{tab:cmcdr_efficiency}
\resizebox{\linewidth}{!}{%
\begin{tabular}{lllrrr}
\toprule
Protocol & Method & Modalities &
Time (min) $\downarrow$ &
Overhead (\%) $\downarrow$ &
$\Delta_M$ (\%) $\downarrow$ \\
\midrule
Replay
& iCaRL
& Video + Acc.
& 46.8
& --
& -- \\

\rowcolor{cmcdrgreen}
Replay
& \textbf{iCaRL + CMCDR}
& Video + Acc.
& 62.9
& +34.4
& -- \\

Replay
& iCaRL
& Video + Acc. + Gyro.
& 57.6
& --
& -- \\

\rowcolor{cmcdrgreen}
Replay
& \textbf{iCaRL + CMCDR}
& Video + Acc. + Gyro.
& 85.4
& +48.3
& +35.8 \\
\midrule
Replay-free
& LwF
& Video + Acc.
& 31.9
& --
& -- \\

\rowcolor{cmcdrgreen}
Replay-free
& \textbf{LwF + CMCDR}
& Video + Acc.
& 43.4
& +36.1
& -- \\

Replay-free
& LwF
& Video + Acc. + Gyro.
& 45.7
& --
& -- \\

\rowcolor{cmcdrgreen}
Replay-free
& \textbf{LwF + CMCDR}
& Video + Acc. + Gyro.
& 69.8
& +52.7
& +60.8 \\
\bottomrule
\end{tabular}}
\end{table}
We report the mean wall-clock training time per incremental stage, measured
using four NVIDIA GeForce RTX 3090 GPUs.
Each baseline and its CMCDR variant use the same architecture, batch size,
training schedule, replay budget, and hardware. The relative overhead is
computed as
$100(t_{\mathrm{CMCDR}}/t_{\mathrm{base}}-1)\%$
under the same modality setting.

We further evaluate how CMCDR's training time scales with the number of
modalities on UESTC-MMEA-CL using iCaRL and LwF.
The two-modality setting uses video and acceleration signals. The
three-modality setting additionally includes gyroscope signals.

On UESTC-MMEA-CL, CMCDR increases training time by $34.4\%\mathord{\sim}52.7\%$
for iCaRL and LwF across the two- and three-modality classification settings.
Increasing the number of modalities from two to three increases the total
training time of iCaRL + CMCDR and LwF + CMCDR by $35.8\%$ and $60.8\%$,
respectively.

\paragraph{Effect of intervention sample ratio.}
We study the efficiency--performance trade-off of CMCDR by varying the
fraction of samples used for modality-subset interventions from 25\% to
100\%, while keeping all other training settings fixed. A lower
intervention ratio reduces the additional forward passes required to
estimate modality contributions. We evaluate replay-based iCaRL and
replay-free LwF in terms of average accuracy, average forgetting, the MCD score,
and average training time per incremental stage.

\begin{figure}[htbp]
    \centering
    \includegraphics[width=\textwidth]{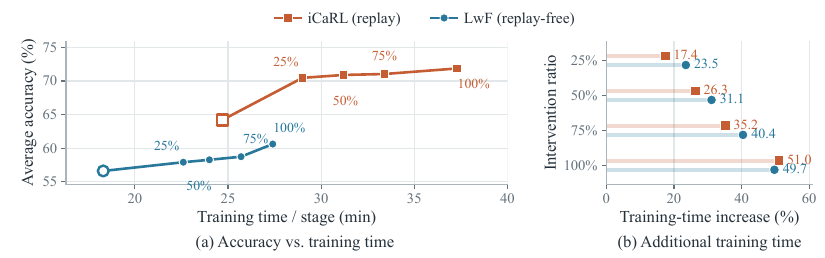}
    \caption{Training-time--accuracy trade-off on AVE using the measurements in
    Table~\ref{tab:probe_ratio_efficiency}. (a) Average accuracy versus mean
    training time per incremental stage. Open markers denote baselines; filled
    markers denote CMCDR, with labels indicating the intervention ratio.
    (b) Relative training-time increase over each baseline.}
    \label{fig:efficiency_tradeoff}
\end{figure}

As shown in Table~\ref{tab:probe_ratio_efficiency}, CMCDR remains effective
with partial interventions. For iCaRL, using only 25\% of samples reduces the MCD score
from 0.779 to 0.452 and improves average accuracy from 64.20\% to 70.46\%,
while limiting the training overhead to 17.4\%.
Increasing the intervention ratio further reduces the MCD score and improves accuracy,
with the lowest forgetting achieved at a 100\% ratio.

A similar trend is observed for LwF. At a 25\% ratio, CMCDR reduces the MCD score from
0.842 to 0.474 and forgetting from 35.11\% to 28.99\%, with a 23.5\%
training overhead. Full intervention further improves performance but increases
the overhead to 49.7\%. Overall, these results show that
CMCDR can reduce its additional overhead through partial interventions,
providing a favorable efficiency--performance trade-off and improving its
scalability to larger continual-learning settings.

\begin{table}[htbp]
\centering
\scriptsize
\setlength{\tabcolsep}{6pt}
\caption{Effect of the intervention sample ratio on AVE.}
\label{tab:probe_ratio_efficiency}
\resizebox{\textwidth}{!}{%
\begin{tabular}{llcccccc}
\toprule
\textbf{Protocol} &
\textbf{Method} &
\textbf{Ratio} &
\textbf{Avg. Acc. (\%) $\uparrow$} &
\textbf{Avg. Forget. (\%) $\downarrow$} &
\textbf{MCD $\downarrow$} &
\textbf{Time / stage (min) $\downarrow$} &
\textbf{Overhead (\%) $\downarrow$} \\
\midrule
Replay & iCaRL & -- & 64.20 & 25.60 & 0.779 & 24.7 & -- \\
\rowcolor{cmcdrgreen}
Replay & iCaRL + CMCDR & 25\% & 70.46 & 18.14 & 0.452 & \textbf{29.0} & +17.4 \\

Replay & iCaRL + CMCDR & 50\% & 70.89 & 17.41 & 0.421 & 31.2 & +26.3 \\
\rowcolor{cmcdrgreen}
Replay & iCaRL + CMCDR & 75\% & 71.03 & 17.72 & 0.402 & 33.4 & +35.2 \\

Replay & iCaRL + CMCDR & 100\% & \textbf{71.86} & \textbf{17.36} & \textbf{0.389} & 37.3 & +51.0 \\
\midrule
Replay-free & LwF & -- & 56.61 & 35.11 & 0.842 & 18.3 & -- \\
\rowcolor{cmcdrgreen}
Replay-free & LwF + CMCDR & 25\% & 57.91 & 28.99 & 0.474 & \textbf{22.6} & +23.5 \\

Replay-free & LwF + CMCDR & 50\% & 58.27 & 28.01 & 0.468 & 24.0 & +31.1 \\
\rowcolor{cmcdrgreen}
Replay-free & LwF + CMCDR & 75\% & 58.73 & 27.88 & 0.433 & 25.7 & +40.4 \\

Replay-free & LwF + CMCDR & 100\% & \textbf{60.61} & \textbf{26.35} & \textbf{0.401} & 27.4 & +49.7 \\
\bottomrule
\end{tabular}}
\end{table}

\subsection{Hyperparameter Sensitivity Analysis}
\label{app:hyperparameter_analysis}

\noindent\textbf{Setup.}
We vary one hyperparameter at a time while keeping all other training and
evaluation settings fixed. We evaluate
$\lambda_{\mathrm{cmcdr}}\in\{0.1,0.5,1.0,2.0\}$ and
$\beta\in\{0.25,0.5,1.0,2.0\}$ on AVE and UESTC-MMEA-CL, using iCaRL for
replay-based training and LwF for replay-free training. For
$\lambda_{\mathrm{cmcdr}}$, we report average accuracy, average forgetting,
the MCD score, and new-task accuracy. For $\beta$, we report average accuracy, average
forgetting, Abs-MCD, and Rel-MCD.

\begin{table}[htbp]
\centering
\scriptsize
\setlength{\tabcolsep}{3.1pt}
\caption{Sensitivity to $\lambda_{\mathrm{cmcdr}}$ on AVE and UESTC-MMEA-CL.}
\label{tab:lambda_cmcdr_sensitivity}
\resizebox{\textwidth}{!}{%
\begin{tabular}{lccccc|cccc}
\toprule
\textbf{Protocol} & $\boldsymbol{\lambda}_{\mathrm{cmcdr}}$ &
\multicolumn{4}{c|}{\textbf{AVE}} &
\multicolumn{4}{c}{\textbf{UESTC-MMEA-CL}} \\
\cmidrule(lr){3-6}\cmidrule(lr){7-10}
& & \textbf{Avg. Acc. $\uparrow$} & \textbf{Avg. Forget. $\downarrow$} &
\textbf{MCD $\downarrow$} & \textbf{New-task Acc. $\uparrow$} &
\textbf{Avg. Acc. $\uparrow$} & \textbf{Avg. Forget. $\downarrow$} &
\textbf{MCD $\downarrow$} & \textbf{New-task Acc. $\uparrow$} \\
\midrule
Replay (iCaRL) & 0.1 & 68.742 & 21.835 & 0.547 & \textbf{76.923} &
73.418 & 20.464 & 0.387 & \textbf{82.369} \\
\rowcolor{cmcdrgreen}
Replay (iCaRL) & 0.5 & \textbf{71.937} & 18.106 & 0.442 & 76.258 &
74.834 & 17.577 & 0.335 & 81.692 \\
Replay (iCaRL) & 1.0 & 71.860 & 17.360 & 0.389 & 75.783 &
\textbf{75.120} & 17.152 & 0.288 & 81.257 \\
\rowcolor{cmcdrgreen}
Replay (iCaRL) & 2.0 & 71.416 & \textbf{16.934} & \textbf{0.342} & 74.605 &
74.707 & \textbf{16.663} & \textbf{0.254} & 79.976 \\
\midrule
Replay-free (LwF) & 0.1 & 58.743 & 31.276 & 0.587 & \textbf{67.425} &
20.158 & 44.724 & 0.428 & \textbf{32.179} \\
\rowcolor{cmcdrgreen}
Replay-free (LwF) & 0.5 & 60.167 & 26.774 & 0.472 & 66.829 &
\textbf{22.913} & 39.116 & 0.353 & 31.568 \\
Replay-free (LwF) & 1.0 & \textbf{60.609} & 26.352 & 0.401 & 65.957 &
22.851 & 38.648 & 0.299 & 30.912 \\
\rowcolor{cmcdrgreen}
Replay-free (LwF) & 2.0 & 60.076 & \textbf{25.713} & \textbf{0.361} & 64.618 &
22.167 & \textbf{37.979} & \textbf{0.262} & 29.715 \\
\bottomrule
\end{tabular}}
\end{table}

\noindent\textbf{Effect of $\lambda_{\mathrm{cmcdr}}$.}
Increasing $\lambda_{\mathrm{cmcdr}}$ from 0.1 to 1.0 reduces the MCD score and
forgetting in all settings, while average accuracy generally improves.
New-task accuracy decreases gradually as the regularization becomes stronger.
At $\lambda_{\mathrm{cmcdr}}=2.0$, the MCD score and forgetting decrease further, but
both average accuracy and new-task accuracy drop across all settings.
Thus, $\lambda_{\mathrm{cmcdr}}=0.5\mathord{\sim}1.0$ provides the best trade-off between
retaining old knowledge and learning new tasks.

\begin{table}[htbp]
\centering
\scriptsize
\setlength{\tabcolsep}{3.1pt}
\caption{Sensitivity to $\beta$ on AVE and UESTC-MMEA-CL.}
\label{tab:beta_sensitivity}
\resizebox{\textwidth}{!}{%
\begin{tabular}{lccccc|cccc}
\toprule
\textbf{Protocol} & $\boldsymbol{\beta}$ &
\multicolumn{4}{c|}{\textbf{AVE}} &
\multicolumn{4}{c}{\textbf{UESTC-MMEA-CL}} \\
\cmidrule(lr){3-6}\cmidrule(lr){7-10}
& & \textbf{Avg. Acc. $\uparrow$} & \textbf{Avg. Forget. $\downarrow$} &
\textbf{Abs-MCD $\downarrow$} & \textbf{Rel-MCD $\downarrow$} &
\textbf{Avg. Acc. $\uparrow$} & \textbf{Avg. Forget. $\downarrow$} &
\textbf{Abs-MCD $\downarrow$} & \textbf{Rel-MCD $\downarrow$} \\
\midrule
Replay (iCaRL) & 0.25 & 70.512 & 18.731 & \textbf{0.223} & 0.201 &
74.326 & 18.427 & \textbf{0.145} & 0.176 \\
\rowcolor{cmcdrgreen}
Replay (iCaRL) & 0.50 & 71.724 & 17.582 & 0.251 & 0.158 &
\textbf{75.342} & 17.618 & 0.174 & 0.135 \\
Replay (iCaRL) & 1.00 & \textbf{71.860} & \textbf{17.360} & 0.281 & 0.129 &
75.120 & \textbf{17.152} & 0.201 & 0.112 \\
\rowcolor{cmcdrgreen}
Replay (iCaRL) & 2.00 & 71.203 & 18.127 & 0.318 & \textbf{0.105} &
74.581 & 17.993 & 0.237 & \textbf{0.087} \\
\midrule
Replay-free (LwF) & 0.25 & 59.604 & 28.246 & \textbf{0.232} & 0.235 &
21.527 & 41.653 & \textbf{0.154} & 0.197 \\
\rowcolor{cmcdrgreen}
Replay-free (LwF) & 0.50 & \textbf{60.795} & 26.827 & 0.264 & 0.176 &
22.381 & 39.412 & 0.194 & 0.162 \\
Replay-free (LwF) & 1.00 & 60.609 & \textbf{26.352} & 0.289 & 0.138 &
\textbf{22.851} & \textbf{38.648} & 0.209 & 0.118 \\
\rowcolor{cmcdrgreen}
Replay-free (LwF) & 2.00 & 59.872 & 27.261 & 0.339 & \textbf{0.112} &
22.054 & 39.786 & 0.257 & \textbf{0.079} \\
\bottomrule
\end{tabular}}
\end{table}

\noindent\textbf{Effect of $\beta$.}
Increasing $\beta$ consistently reduces Rel-MCD but increases Abs-MCD across
both datasets and protocols. Small $\beta$ values better preserve contribution values, whereas large values better preserve the normalized contribution profile. The default setting, $\beta=1.0$, yields the lowest average forgetting in all four settings,
while average accuracy peaks at $\beta=0.5$ or $1.0$.

\subsection{Replay-Free Probe Transfer}
\label{app:probe_transfer}

Proposition~\ref{prop:replay_free_drift_bound} bounds old-task contribution
drift using the current-task probe loss and a distribution-shift term.
We examine whether held-out current-task probes track old-task drift by
comparing class-wise MCD scores using the replay-free LwF.
Lemma~\ref{lem:mcd_controls_contribution_shift} shows that each MCD score
bounds the $\ell_1$ change in the corresponding contribution profile.

After each incremental stage $t\geq 2$, we compare the frozen model
$F_{\theta_{t-1}}$ and the updated model $F_{\theta_t}$ under the same
modality-subset interventions. For two profiles defined as in
Eq.~\ref{eq:cmcdr_theory_profiles}, the per-profile MCD score is
\begin{equation}
d^{\mathrm{MCD}}_{i,r\to t}
=
|m_t-m_r|
+
(\min(m_t,m_r)+\epsilon)
\|\widehat{\mathbf c}_t-\widehat{\mathbf c}_r\|_1.
\label{eq:probe_mcd_score}
\end{equation}
For the comparison from $t-1$ to $t$, we denote this score for sample
$\mathbf x_i$ and target class $c$ by
$d_{t-1\to t}^{\mathrm{MCD}}(\mathbf x_i,c)$.
For each old class $c\in\mathcal Y_{<t}$, we compute
\begin{equation}
D_{\mathrm{cur}}^{t,c}
=
\frac{1}{|\mathcal Q_t|}
\sum_{\mathbf x_i\in\mathcal Q_t}
d_{t-1\to t}^{\mathrm{MCD}}(\mathbf x_i,c),
\qquad
D_{\mathrm{old}}^{t,c}
=
\frac{1}{|\mathcal P_c^{\mathrm{class}}|}
\sum_{\mathbf x_i\in\mathcal P_c^{\mathrm{class}}}
d_{t-1\to t}^{\mathrm{MCD}}(\mathbf x_i,c),
\label{eq:probe_transfer}
\end{equation}
where $\mathcal Q_t$ contains held-out current-task probes and
$\mathcal P_c^{\mathrm{class}}$ contains held-out probes from old class $c$.
Both scores use $c$ as the target class.

At each stage, we compute the Spearman correlation between
$D_{\mathrm{cur}}^{t,c}$ and $D_{\mathrm{old}}^{t,c}$ across old classes.
We report the arithmetic mean of these within-stage correlations across
incremental stages.
We also measure the mean absolute MCD gap across old classes:
\begin{equation}
G_t^{\mathrm{MCD}}
=
\frac{1}{|\mathcal Y_{<t}|}
\sum_{c\in\mathcal Y_{<t}}
\left|
D_{\mathrm{old}}^{t,c}-D_{\mathrm{cur}}^{t,c}
\right|.
\label{eq:empirical_transfer_gap}
\end{equation}
We report the stage-averaged gap $\overline{G}^{\mathrm{MCD}}$.
At each stage, we normalize $G_t^{\mathrm{MCD}}$ by the mean of
$D_{\mathrm{old}}^{t,c}$ over old classes.
We report the average normalized gap across stages as a percentage.

\begin{figure}[htbp]
\centering
\includegraphics[width=\linewidth]{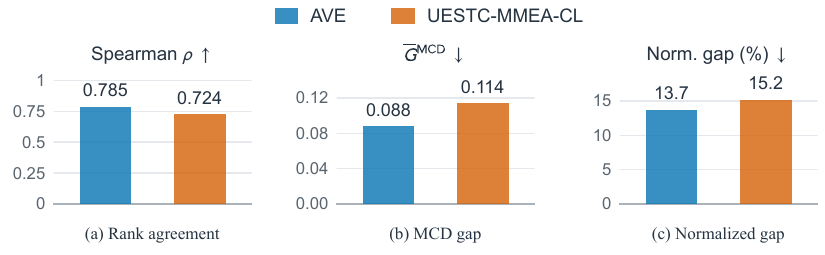}
\caption{Agreement between class-wise MCD scores on held-out current-task
and old-class probes using the replay-free LwF:
(a) mean within-stage Spearman correlation.
(b) stage-averaged mean absolute MCD gap $\overline{G}^{\mathrm{MCD}}$.
(c) average normalized gap.}
\label{fig:probe_transfer}
\end{figure}

Figure~\ref{fig:probe_transfer} reports mean within-stage Spearman
correlations of $0.785$ and $0.724$ on AVE and UESTC-MMEA-CL, respectively.
The stage-averaged mean absolute MCD gaps are $0.088$ and $0.114$,
with normalized gaps of $13.7\%$ and $15.2\%$.
These results support using current-task probes to estimate old-class MCD
in replay-free learning.

\subsection{Ablation and Controlled Comparisons}
\label{app:contribution_ablation}

\noindent\textbf{Loss components.}
We first ablate the two terms in the CMCDR objective. Let
$\widehat{\mathbf C}=\mathbf C/(\|\mathbf C\|_1+\epsilon)$ and define
$\widehat{\mathbf C}^{-}$ analogously. The two components are
\[
\mathcal L_{\mathrm{abs}}
=\rho\!\left(\mathbf C,\operatorname{sg}(\mathbf C^{-})\right),
\qquad
\mathcal L_{\mathrm{rel}}
=\rho\!\left(\widehat{\mathbf C},\operatorname{sg}(\widehat{\mathbf C}^{-})\right).
\]
Equation~\ref{eq:cmcdr_profile_distance} therefore gives
$d_{\mathrm{cp}}=\mathcal L_{\mathrm{abs}}+\beta\mathcal L_{\mathrm{rel}}$.
The relative term retains its weight $\beta$ when used alone.
All variants use the full M\"obius contribution profile with identical
training settings.

\begin{table}[htbp]
\centering
\scriptsize
\setlength{\tabcolsep}{4pt}
\caption{Ablation of the CMCDR loss components.}
\label{tab:contribution_profile_ablation}
\resizebox{\textwidth}{!}{%
\begin{tabular}{lcc|cc|cc}
\toprule
& \multicolumn{2}{c|}{\textbf{Components}} &
\multicolumn{2}{c|}{\textbf{AVE}} &
\multicolumn{2}{c}{\textbf{UESTC-MMEA-CL}} \\
\textbf{Protocol} &
$\mathcal L_{\mathrm{abs}}$ &
$\mathcal L_{\mathrm{rel}}$ &
\textbf{Avg. Acc. $\uparrow$} &
\textbf{Avg. Forget. $\downarrow$} &
\textbf{Avg. Acc. $\uparrow$} &
\textbf{Avg. Forget. $\downarrow$} \\
\midrule
Replay (iCaRL) & $\times$ & $\times$ & 64.20 & 25.60 & 71.83 & 22.86 \\
Replay (iCaRL) & $\checkmark$ & $\times$ & 69.27 & 20.38 & 74.07 & 19.17 \\
Replay (iCaRL) & $\times$ & $\checkmark$ & 68.11 & 21.57 & 73.48 & 20.05 \\
\rowcolor{cmcdrgreen}
Replay (iCaRL) & $\checkmark$ & $\checkmark$ &
\textbf{71.86} & \textbf{17.36} &
\textbf{75.12} & \textbf{17.15} \\
\midrule
Replay-free (LwF) & $\times$ & $\times$ & 56.61 & 35.11 & 17.13 & 49.37 \\
Replay-free (LwF) & $\checkmark$ & $\times$ & 59.46 & 29.18 & 21.36 & 42.09 \\
Replay-free (LwF) & $\times$ & $\checkmark$ & 58.91 & 30.43 & 20.83 & 43.20 \\
\rowcolor{cmcdrgreen}
Replay-free (LwF) & $\checkmark$ & $\checkmark$ &
\textbf{60.61} & \textbf{26.35} &
\textbf{22.85} & \textbf{38.65} \\
\bottomrule
\end{tabular}}
\end{table}

Both components improve accuracy and reduce forgetting when used independently.
When used alone, absolute matching achieves higher accuracy and lower forgetting than relative matching. Combining both terms yields the best performance across both datasets and learning protocols.

\medskip
\noindent\textbf{Contribution-profile parameterization.}
We next examine whether retaining modality interactions and explicitly
representing them with M\"obius coordinates are beneficial. For each old-class
index $k$, we define the centered subset-response vector as
\begin{equation}
\overline{\mathbf B}_{i,s}^{k}
=
\left[
B_{i,s}^{k,\mathcal Y_{<t}}(\mathcal S)
-
B_{i,s}^{k,\mathcal Y_{<t}}(\emptyset)
\right]_{\emptyset\neq\mathcal S\subseteq\mathcal M}.
\label{eq:centered_subset_profile_ablation}
\end{equation}
The centered subset-response vector and the full M\"obius profile are related
by an invertible linear transform and therefore retain the same subset-level
information. We compare three parameterizations: singleton-only contributions,
which exclude interaction terms; centered subset responses, which retain
interactions implicitly; and the full M\"obius profile, which represents
interaction terms explicitly. All variants evaluate the same $2^M$ modality
subsets and use both terms in Eq.~\ref{eq:cmcdr_profile_distance}. The
regularization coefficients are selected using the same search grid and tuning
budget.

\begin{table}[htbp]
\centering
\scriptsize
\setlength{\tabcolsep}{3.8pt}
\caption{Ablation of the contribution-profile parameterization.}
\label{tab:mobius_parameterization_ablation}
\resizebox{\textwidth}{!}{%
\begin{tabular}{lllcc|cc}
\toprule
\textbf{Protocol} &
\textbf{Parameterization} &
\textbf{Interactions} &
\multicolumn{2}{c|}{\textbf{AVE}} &
\multicolumn{2}{c}{\textbf{UESTC-MMEA-CL}} \\
\cmidrule(lr){4-5}\cmidrule(lr){6-7}
& & &
\textbf{Avg. Acc. $\uparrow$} &
\textbf{Avg. Forget. $\downarrow$} &
\textbf{Avg. Acc. $\uparrow$} &
\textbf{Avg. Forget. $\downarrow$} \\
\midrule
Replay (iCaRL) &
Singleton-only & Excluded &
68.52 & 21.36 & 73.41 & 20.02 \\

Replay (iCaRL) &
Centered subsets & Implicit &
70.13 & 18.89 & 73.96 & 18.84 \\

\rowcolor{cmcdrgreen}
Replay (iCaRL) &
M\"obius & Explicit &
\textbf{71.86} & \textbf{17.36} &
\textbf{75.12} & \textbf{17.15} \\
\midrule
Replay-free (LwF) &
Singleton-only & Excluded &
58.74 & 31.02 & 20.11 & 44.27 \\

Replay-free (LwF) &
Centered subsets & Implicit &
59.46 & 28.41 & 21.53 & 40.66 \\

\rowcolor{cmcdrgreen}
Replay-free (LwF) &
M\"obius & Explicit &
\textbf{60.61} & \textbf{26.35} &
\textbf{22.85} & \textbf{38.65} \\
\bottomrule
\end{tabular}}
\end{table}

Both parameterizations that retain interaction information achieve
higher accuracy and lower forgetting than singleton-only matching.
The full M\"obius profile improves average accuracy by 1.15$\mathord{\sim}$1.73 percentage
points and reduces forgetting by 1.53$\mathord{\sim}$2.06 points over centered-response
matching across both datasets and learning protocols. M\"obius matching
penalizes each contribution change separately. For two modalities,
its absolute loss includes an additional interaction penalty whenever
singleton changes offset interaction drift and preserve the joint centered
response (Proposition~\ref{prop:interaction_regularization}).

\medskip
\noindent\textbf{Controlled comparison with subset-wise distillation.}
We compare CMCDR with subset-wise distillation (SubsetKD) using identical
intervention samples, all $2^M$ modality subsets, and the same forward passes
through the current and frozen models. SubsetKD matches old-class output
distributions, whereas CMCDR matches contribution profiles.

For stage $s\in\{t-1,t\}$, the old-class distribution under modality subset
$\mathcal S$ is
\begin{equation}
\mathbf p_{i,s}^{<t}(\mathcal S;T)
=
\operatorname{softmax}
\left(
\frac{
[z_{i,s,k}(\mathcal S)]_{k\in\mathcal Y_{<t}}
}{T}
\right),
\label{eq:subset_kd_distribution}
\end{equation}
where $T$ is the distillation temperature. Let $\mathcal I_t$ denote the
intervention inputs, with
$\mathcal I_t=\{\mathbf x_i:(\mathbf x_i,y_i)\in\mathcal X_{\mathrm{old}}\}$
for replay-based training and
$\mathcal I_t=\mathcal X_{\mathrm{new}}$ for replay-free training.
SubsetKD is defined as
\begin{equation}
\mathcal L_{\mathrm{subKD}}
=
\frac{T^2}{|\mathcal I_t|\,2^M}
\sum_{\mathbf x_i\in\mathcal I_t}
\sum_{\mathcal S\subseteq\mathcal M}
\operatorname{KL}
\left(
\operatorname{sg}\!\left[
\mathbf p_{i,t-1}^{<t}(\mathcal S;T)
\right]
\middle\|
\mathbf p_{i,t}^{<t}(\mathcal S;T)
\right).
\label{eq:subset_kd}
\end{equation}
We tune $T\in\{1,2,4,8\}$ and
$\lambda_{\mathrm{subKD}}\in\{0.1,0.5,1.0,2.0\}$, using the same
weight grid as for $\lambda_{\mathrm{cmcdr}}$.
The configuration $T=2$ and $\lambda_{\mathrm{subKD}}=1$ performs best
in both protocols and is used for the results in Table~\ref{tab:subset_kd_control}.

The objectives differ even when evaluated on the same subsets. Adding a
subset-independent constant to a class benefit leaves its centered responses
and M\"obius contributions unchanged, but can change its output probability.
For example, let the frozen and current binary logits be
$(v(\mathcal S),0)$ and $(v(\mathcal S)+q,0)$, respectively, with
$q\neq0$ shared across subsets. The two class benefits shift by $q$ and
$-q$, so both contribution profiles remain unchanged. CMCDR has zero loss,
whereas SubsetKD has positive loss because the class probabilities change.

\begin{table}[htbp]
\centering
\scriptsize
\setlength{\tabcolsep}{3.8pt}
\caption{Controlled comparison between CMCDR and subset-wise distillation (SubsetKD).}
\label{tab:subset_kd_control}
\resizebox{\textwidth}{!}{%
\begin{tabular}{lllccc|cc}
\toprule
\textbf{Protocol} &
\textbf{Method} &
\textbf{Preservation target} &
\textbf{\# Subsets} &
\multicolumn{2}{c|}{\textbf{AVE}} &
\multicolumn{2}{c}{\textbf{UESTC-MMEA-CL}} \\
\cmidrule(lr){5-6}\cmidrule(lr){7-8}
& & & &
\textbf{Avg. Acc. $\uparrow$} &
\textbf{Avg. Forget. $\downarrow$} &
\textbf{Avg. Acc. $\uparrow$} &
\textbf{Avg. Forget. $\downarrow$} \\
\midrule
Replay &
iCaRL &
Full-input logits &
$1$ &
64.20 & 25.60 & 71.83 & 22.86 \\

Replay &
iCaRL + SubsetKD &
Subset-wise distributions &
$2^M$ &
66.74 & 23.63 & 72.46 & 21.61 \\

\rowcolor{cmcdrgreen}
Replay &
iCaRL + CMCDR &
M\"obius contribution profile &
$2^M$ &
\textbf{71.86} & \textbf{17.36} &
\textbf{75.12} & \textbf{17.15} \\
\midrule
Replay-free &
LwF &
Full-input distillation &
$1$ &
56.61 & 35.11 & 17.13 & 49.37 \\

Replay-free &
LwF + SubsetKD &
Subset-wise distributions &
$2^M$ &
57.89 & 33.78 & 19.86 & 46.57 \\

\rowcolor{cmcdrgreen}
Replay-free &
LwF + CMCDR &
M\"obius contribution profile &
$2^M$ &
\textbf{60.61} & \textbf{26.35} &
\textbf{22.85} & \textbf{38.65} \\
\bottomrule
\end{tabular}}
\end{table}

SubsetKD improves accuracy and reduces forgetting over the underlying
continual-learning baselines. Across both datasets and protocols, CMCDR
improves average accuracy by $2.66\mathord{\sim}5.12$ percentage points and reduces
average forgetting by $4.46\mathord{\sim}7.92$ percentage points relative to SubsetKD
(Table~\ref{tab:subset_kd_control}).

\subsection{Robustness to Alternative Contribution Estimators}
\label{app:alternative_contribution_estimators}

We finally examine whether the observed contribution stabilization depends on
the M\"obius-based estimator used for the MCD score. We re-evaluate the same trained models
using CMoB~\citep{wang2025cmob} and Shapley-based modality
valuation~\citep{wei2024samplevaluation}. Both estimators are used only for
post-hoc evaluation. For each estimator, we construct contribution vectors at each stage and
compute their drift using the distance in Eq.~\ref{eq:mcd}.
We call the resulting measures CMoB-value drift (CVD) and
Shapley-value drift (SVD).

\begin{figure}[htbp]
\centering
\includegraphics[width=\linewidth]{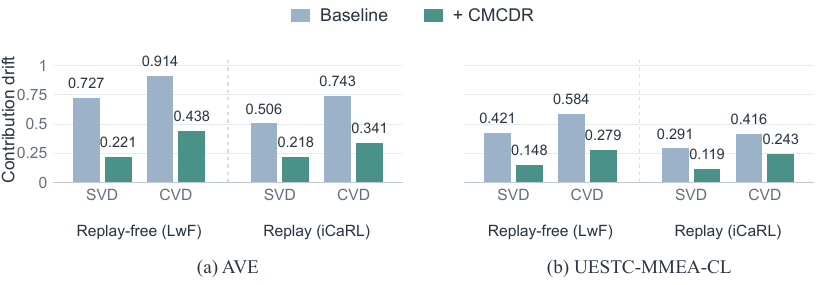}
\caption{Contribution drift under Shapley-based (SVD) and CMoB-based (CVD) estimators on (a) AVE and (b) UESTC-MMEA-CL.}
\label{fig:contribution_estimator_evaluation}
\end{figure}

As shown in Figure~\ref{fig:contribution_estimator_evaluation}, CMCDR reduces contribution drift across both datasets and learning protocols
under the Shapley- and CMoB-based estimators. Together with the MCD results,
these findings demonstrate consistent contribution stabilization across three
estimators.

\section{Comparison with Related Methods}
\label{app:comparison_distinctiveness}

We discuss how CMCDR differs from existing methods in multimodal continual learning and balanced multimodal learning, and compare it with representative modality-balancing methods under MMCL settings.

\subsection{Multimodal Continual Learning}
\label{app:mmcl_relationship}

Multimodal continual learning (MMCL) learns sequential tasks from multiple
modalities while retaining previously acquired knowledge. Existing methods
mainly reduce forgetting by preserving multimodal representations.
AV-CIL~\citep{pian2023avcil} maintains instance- and class-level audio--visual
semantic similarity and distills audio-guided visual attention.
CIGN~\citep{mo2023cign} distills audio--visual class tokens to preserve
class-aware cross-modal representations.
MAFED~\citep{nikandrou2024mafed} applies modality-aware distillation to visual
and question representations.
MG-CLIP~\citep{huang2025mindgap} preserves the image--text modality gap in
CLIP-based continual learning.
These methods regularize feature similarity, token representations, attention,
or cross-modal relations. CMCDR instead constrains changes in old-task modality
contribution profiles, including both individual modalities and their
interactions.

\subsection{Balanced Multimodal Learning}
\label{app:balanced_multimodal_relationship}

Balanced multimodal learning addresses modality dominance and under-optimization
during joint multimodal training
\citep{peng2022ogm,fan2023pmr,li2023agm,wei2024mmpareto,hua2024reconboost}.
OGM-GE~\citep{peng2022ogm} adjusts modality-specific gradients according to
their optimization states.
Shapley-based modality valuation estimates sample-level modality contributions
and strengthens low-contributing modalities
\citep{wei2024samplevaluation}.
CMoB~\citep{wang2025cmob} estimates modality contributions through causal
effects and selectively improves under-optimized modalities.

In MMCL, modality contributions are task-dependent.
As illustrated in Figure~\ref{motivation}, different tasks can rely on different
modalities and modality interactions. A low modality contribution therefore
does not necessarily indicate under-optimization.
Modality-balancing methods adjust how modalities are optimized for the current
task. CMCDR instead constrains changes in old-task contribution profiles for
individual modalities and their interactions across incremental stages,
rather than encouraging a common contribution pattern across tasks.

\subsection{Comparison with Modality-Balancing Methods}
\label{app:modality_balancing_comparison}
\begin{figure}[htbp]
\centering
\includegraphics[width=\textwidth]{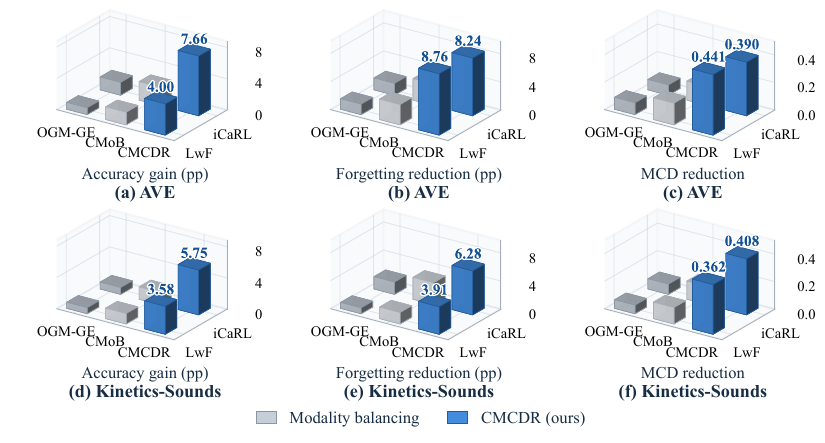}
\caption{Comparison with modality-balancing methods on AVE and Kinetics-Sounds. Bars show accuracy gains (pp), forgetting reductions (pp), and MCD reductions relative to the corresponding LwF or iCaRL baseline. Higher is better in all panels.}
\label{fig:method_comparison_3d}
\end{figure}
We integrate OGM-GE~\citep{peng2022ogm} and
CMoB~\citep{wang2025cmob} into replay-free and replay-based MMCL pipelines.
Under the replay-free protocol, both methods operate on the current-task
multimodal objective while the original LwF~\citep{li2017lwf} distillation
objective remains unchanged. Replay-free CMCDR uses current-task probes and the
frozen previous model to preserve old-class contribution profiles.
Under the replay-based protocol, OGM-GE and CMoB are integrated into the same
iCaRL training pipeline used by CMCDR.
Within each protocol, all methods use the same backbone, classifier, task order,
data augmentation, optimizer, memory budget, and number of updates.
CMCDR uses the same contribution-profile construction and coefficient as in the
main experiments.

Figure~\ref{fig:method_comparison_3d} shows that OGM-GE and CMoB
improve accuracy and reduce forgetting over the corresponding LwF and iCaRL
baselines. They also reduce the MCD score, indicating that more balanced multimodal
optimization can partially alleviate contribution drift. However, multimodal balancing methods generally encourage under-optimized
modalities to contribute more during the current task, whereas different tasks
can legitimately exhibit different modality dependencies. As a result, enforcing modality balance across tasks can distort the
modality dependence learned for earlier tasks. CMCDR instead preserves the task-specific modality dependence of old tasks,
leading to lower MCD scores, higher accuracy, and less forgetting across both datasets
and protocols. 

Overall, modality-balancing methods adjust modality optimization for the current
task without explicitly preserving old-task contribution profiles.
When old and new tasks rely on different modalities, these updates may
disrupt the contributions supporting old-task predictions.
CMCDR instead constrains these changes by matching old-class contribution
profiles to those of the frozen previous model.


\end{document}